\def\year{2021}\relax
\documentclass[letterpaper]{article} 
\usepackage{aaai21}  
\usepackage{times}  
\usepackage{helvet} 
\usepackage{courier}  
\usepackage[hyphens]{url}  
\usepackage{graphicx} 
\usepackage{natbib}  
\usepackage{caption} 
\usepackage{amsthm}

\newtheorem{assumption}{Assumption}
\newtheorem{lemma}{Lemma}
\newtheorem{theorem}{Theorem}

\newtheorem{corollary}{Corollary}
\newtheorem{proposition}{Proposition}
\usepackage{amsfonts}
\usepackage{mathtools}
\usepackage{amsmath}
\DeclareMathOperator{\sign}{sign}
\usepackage{algcompatible}
\usepackage{algorithm}
\usepackage{color}
\usepackage{booktabs}
\usepackage[switch]{lineno}

\title{Step-Ahead Error Feedback\\ for Distributed Training with Compressed Gradient}
\author {
    An Xu,\textsuperscript{\rm 1}
    Zhouyuan Huo,\textsuperscript{\rm 2}
    Heng Huang\textsuperscript{\rm 1,3} \\
}
\affiliations {
    \textsuperscript{\rm 1} Electrical and Computer Engineering Department, University of Pittsburgh, PA, USA \\
    \textsuperscript{\rm 2} Google, Mountain View, CA, USA\\
    \textsuperscript{\rm 3} JD Finance America Corporation, Mountain View, CA, USA \\
    \{an.xu, heng.huang\}@pitt.edu, zhouyuan.huo@gmail.com
}

\begin{document}

\maketitle

\begin{abstract}
Although the distributed machine learning methods can speed up the training of large deep neural networks, the communication cost has become the non-negligible bottleneck to constrain the performance. To address this challenge, the gradient compression based communication-efficient distributed learning methods were designed to reduce the communication cost, and more recently the local error feedback was incorporated to compensate for the corresponding performance loss. However, in this paper, we will show that a new ``gradient mismatch" problem is raised by the local error feedback in centralized distributed training and can lead to degraded performance compared with full-precision training. To solve this critical problem, we propose two novel techniques, 1) step ahead and 2) error averaging, with rigorous theoretical analysis. Both our theoretical and empirical results show that our new methods can handle the ``gradient mismatch" problem. The experimental results show that we can even \textbf{train faster with common gradient compression} schemes than both the full-precision training and local error feedback \textbf{regarding the training epochs and without performance loss}.
\end{abstract}

\section{Introduction}\label{introduction}
Distributed training is a common practice in training large models with big datasets. The master-slave architecture is the most common paradigm in centralized learning, where the worker nodes compute gradients based on the local dataset and communicate with the server node. While in decentralized learning~\cite{DBLP:conf/icml/LianZZL18,lian2017can,tang2018communication,pmlr-v80-tang18a}, no server node is needed and each worker node only communicates with its neighbors to avoid the heavy traffic of the server node as in centralized training. In recent research, the gradient compression techniques have been widely used to reduce the communication cost in both centralized and decentralized training.

Mild gradient compression techniques~\cite{alistarh2017qsgd,xu2020optimal} offer a mild compression ratio at the cost of negligible performance loss. However, it is more attractive to use an aggressive compression technique such as SignSGD~\cite{pmlr-v80-bernstein18a} which is even more favorable for scaling up the number of worker nodes. More recently, the local error feedback method \cite{karimireddy2019error} was introduced to fix the corresponding non-negligible performance loss resulting from SignSGD via adding the compression error at the current iteration to the next iteration. Note that for SignSGD, we need to scale it by a factor before applying the local error feedback method because it does not satisfy the Assumption \ref{compressor} described at later section, which is crucial in the theoretical analysis of local error feedback.

Methods other than gradient compression to accelerate distributed training include asynchronous methods~\cite{lian2015asynchronous,ho2013more,huang2019tangram,xu2020acceleration}, local SGD~\cite{stich2018local} which is also a natural fit for solving federated learning~\cite{konevcny2016federated} problem, and communication scheduling such as the lazy aggregation of gradients \cite{sun2019communication,hashemi2018tictac,chen2018lag}. Specifically, asynchronous distributed training avoids the synchronization barrier and the worker node does not wait for each other. The performance loss is related to the inconsistency between the worker and server (staleness) allowed during training. In local SGD, each worker node stores a copy of the model and does several iterations of updating before communicating with all other nodes to average the updated model. The more the number of local updating iterations is, the more the model in different worker nodes will diverge, leading to a larger performance loss.

We focus on the line of works with gradient compression. Plain gradient compression has been well studied both in centralized \cite{alistarh2017qsgd,wen2017terngrad,stich2018sparsified} and decentralized training \cite{tang2018communication,koloskova2019decentralized}. Later works incorporating local error feedback theoretically and empirically achieve superior performance in centralized \cite{basu2019qsparse,wu2018error,zheng2019communication} and decentralized training \cite{tang2019deepsqueeze} than plain gradient compression. In this paper, we improve the local error feedback with theoretical analysis and empirical validation in centralized distributed training. When we studied the coarse idea of adding the current compression error to the next iteration as local error feedback does, we found that this strategy could lead to a one-iteration outdated gradient. This staleness may seem trivial at first glance, but theory and practice show that it can be the reason why local error feedback is not always lossless. We summarize the main contributions of our paper as follows:
\begin{itemize}
    \item We introduce and discuss the new ``gradient mismatch" problem caused by the local error feedback with the potential to lead to stale gradients. We show that the local error feedback may not be able to achieve lossless performance all the time in experiments. To the best of our knowledge, this is the first paper to systematically investigate this problem.
    \item We propose two novel techniques, 1) step ahead and 2) error averaging, to correct the ``gradient mismatch" issue. Error averaging can be conducted in a much more infrequent way than the communication of the compressed gradient.
    \item Theoretical analysis shows a better error bound of our proposed method than local error feedback. Experimental results verify that our method converges even \textit{faster with common gradient compression regarding training epochs without performance loss} compared with both the full-precision training and local error feedback.
\end{itemize}

\section{Local Error Feedback}\label{local error feedback}

We consider the following learning problem:
\begin{equation}
    \min_{\textbf{x}}F(\textbf{x})\coloneqq \mathbb{E}_{\xi\sim\mathcal{D}} f(\textbf{x};\xi)\,,
\end{equation}
where $\textbf{x}$ is the parameters, $F(\cdot)$ is the full loss function, $\mathcal{D}$ is the data distribution, $\xi$ is the random variable associated with stochastic sampling and $f(\cdot)$ is the loss function associated with certain data sample. Stochastic optimization methods compute the stochastic gradient $\nabla f(\textbf{x};\xi)$ to update $\textbf{x}$. We assume the data distributions across different workers are identical.

In local error feedback, the compression error is added into the next iteration of training. We illustrate it in Algorithm \ref{error feedback algorithm} (line 14 $\sim$ 16). In the first work~\cite{karimireddy2019error} using local error feedback to fix the performance loss resulting from scaled SignSGD~\cite{pmlr-v80-bernstein18a} compression, only SGD rather than momentum SGD is considered (momentum constant $\mu=0$). \cite{zheng2019communication} proposed to use local error feedback to fix momentum SGD with block-wise scaled SignSGD compression. In \cite{zheng2019communication}, the feedbacked error $\textbf{e}^{(k)}_t$ is scaled according to learning rate as $\frac{\eta_{t-1}}{\eta_t}\textbf{e}^{(k)}_t$. Typically, scaled SignSGD compresses a vector $\textbf{v}\in \mathbb{R}^d$ to
\begin{equation}
    \mathcal{C}(\textbf{v})=\frac{\|\textbf{v}\|_1}{d}\sign(\textbf{v})\,.
\end{equation}

For simplicity, we refer to scaled SignSGD as SignSGD from now on. To put it in a more general and clearer way as in Algorithm \ref{error feedback algorithm}, we compress the local model difference $\Delta^{(k)}_{t+1}$ after updating and re-update the local model with the information $\mathcal{C}(\Delta_{t+1})$ that the server has gathered from all the workers and compressed. There are two advantages: 1) it can be easily extended to local SGD, where the local model difference will be communicated every several ($>1$) iterations; 2) the local error needn't be scaled when using a decaying learning rate. Note that in Algorithm \ref{error feedback algorithm} where the local model difference is communicated every iteration, \textbf{the local model $\textbf{x}^{(k)}_t=\textbf{x}_t$ is identical across all workers}. We do not need to synchronize the local model at every iteration.

\textbf{Gradient Mismatch.} The effectiveness of local error feedback comes from an auxiliary variable $\tilde{\textbf{x}}_t\coloneqq \textbf{x}_t-(\textbf{e}_t+\frac{1}{K}\sum^{K}_{k=1}\textbf{e}^{(k)}_t)$ in its theoretical analysis. Although a very aggressive gradient compression scheme may be applied, the update of the auxiliary variable in local error feedback still satisfies:
\begin{equation}\label{local error feedback auxiliary update}
    \tilde{\textbf{x}}_{t+1}=\tilde{\textbf{x}}_t-\frac{\eta_t}{K}\sum^{K}_{k=1}\textbf{m}^{(k)}_{t+1}\,.
\end{equation}

For vanilla momentum SGD, we update the parameters $\textbf{x}\leftarrow \textbf{x}-\frac{\eta_t}{K}\sum^{K}_{k=1}\textbf{m}^{(k)}_{t+1}$ at iteration $t$. While in local error feedback, the auxiliary variable $\tilde{\textbf{x}}_t$ is updated in a similar way shown by Eq.~(\ref{local error feedback auxiliary update}). We refer to $(\textbf{e}_t+\frac{1}{K}\sum^{K}_{k=1}\textbf{e}^{(k)}_t)$ as the compression error term which is usually trivial at the end of training. The small compression error makes the output parameters $\textbf{x}_T$ similar to the auxiliary variable $\tilde{\textbf{x}}_T$. However, is the auxiliary variable $\tilde{\textbf{x}}_T$ the same as the training results of vanilla momentum SGD? The answer is ``no" due to a slight but important difference: \textbf{the momentum term $\textbf{m}^{(k)}_{t+1}$ is computed based on the gradient of $\textbf{x}_t$} ($\nabla f(\textbf{x}^{(k)}_t;\xi^{(k)}_t)$ in Algorithm \ref{error feedback algorithm} line 14) \textbf{but used to update} $\tilde{\textbf{x}}_{t}$. We name it as the ``gradient mismatch" problem, which can jeopardize the scalability of local error feedback in various tasks and models.

\begin{algorithm}[!t]
\caption{Distributed Momentum SGD with Double-Way Compression.}
   \label{error feedback algorithm}
\begin{algorithmic}[1]
    \STATE {\bfseries Input:} averaging period $p>1$, number of iterations $T$, number of workers $K$, learning rate $\{\eta_t\}_{t=0}^{T-1}$, parameters $\textbf{x}_0$, compression scheme $\mathcal{C}(\cdot)$ and the momentum constant $0\leq \mu<1$.
    \STATE {\bfseries Initialize:} $\forall 1\leq k \leq K$, initial local parameters $\textbf{x}^{(k)}_0=\textbf{x}_0$ and local error $\textbf{e}^{(k)}_0=\textbf{0}$ and local momentum buffer $\textbf{m}^{(k)}_0=\textbf{0}$. $\textbf{x}^{(k)}_t=\textbf{x}_t$ for all $t=0,\cdots,T$.
    
    \FOR{$t=0,\cdots,T-1$}
    \STATE {\bfseries Worker-$k$:}
    
    \IF{\textbf{\textit{Step Ahead Error Feedback (SAEF)}}}
        \IF {$\mod(t+1,p)=0$}
            \STATE Average local error $\textbf{e}^{(k)}_t\leftarrow \frac{1}{K}\sum^{K}_{k=1}\textbf{e}^{(k)}_t$
        \ENDIF
        \STATE $\textbf{x}^{(k)}_{t+\frac{1}{2}}=\textbf{x}^{(k)}_t-\textbf{e}^{(k)}_t$  \hfill\emph{// One step ahead.}
        \STATE $\textbf{m}^{(k)}_{t+1}=\mu \textbf{m}^{(k)}_{t}+\nabla f(\textbf{x}^{(k)}_{t+\frac{1}{2}};\xi^{(k)}_t)$
        \STATE $\textbf{x}^{(k)}_{t+1}=\textbf{x}^{(k)}_{t+\frac{1}{2}}-\eta_t \textbf{m}^{(k)}_{t+1}$ \hfill\emph{// Momentum SGD update.}
        \STATE $\Delta^{(k)}_{t+1}=\textbf{e}^{(k)}_t+\textbf{x}^{(k)}_{t+\frac{1}{2}}-\textbf{x}^{(k)}_{t+1}$
    \ELSIF{\textbf{\textit{Local Error Feedback (EF)}}}
        \STATE $\textbf{m}^{(k)}_{t+1}=\mu \textbf{m}^{(k)}_{t}+\nabla f(\textbf{x}^{(k)}_t;\xi^{(k)}_t)$
        \STATE $\textbf{x}^{(k)}_{t+1}=\textbf{x}^{(k)}_t-\eta_t \textbf{m}^{(k)}_{t+1}$ \hfill\emph{// Momentum SGD update.}
        \STATE $\Delta^{(k)}_{t+1}=\textbf{e}^{(k)}_t+\textbf{x}^{(k)}_t-\textbf{x}^{(k)}_{t+1}$
    \ENDIF
    
    \STATE $\textbf{e}^{(k)}_{t+1}=\Delta^{(k)}_{t+1}-\mathcal{C}(\Delta^{(k)}_{t+1})$
    \STATE Send $\mathcal{C}(\Delta^{(k)}_{t+1})$ to the server node.
    
    \STATEx
    \STATE {\bfseries Server:}
    \STATE $\Delta_{t+1}=\textbf{e}_t+\frac{1}{K}\sum^{K}_{k=1}\mathcal{C}(\Delta^{(k)}_{t+1})$
    \STATE $\textbf{e}_{t+1}=\Delta_{t+1}-\mathcal{C}(\Delta_{t+1})$
    \STATE Broadcast $\mathcal{C}(\Delta_{t+1})$ to all the worker nodes.
    
    \STATEx
    \STATE {\bfseries Worker-$k$:}
    \STATE $\textbf{x}^{(k)}_{t+1}=\textbf{x}^{(k)}_t-\mathcal{C}(\Delta_{t+1})$ \hfill\emph{// Re-update.}
   \ENDFOR
   \STATE {\bfseries Output:} parameters $\textbf{x}_T=\textbf{x}^{(k)}_T$
\end{algorithmic}
\end{algorithm}

\section{Resolving Gradient Mismatch Problem}\label{fix gradient mismatch}

To alleviate the effect of gradient mismatch, we propose a new step-ahead local error-feedback (SAEF) algorithm as summarized in Algorithm \ref{error feedback algorithm} (line 6 $\sim$ 12). According to the update of the auxiliary variable Eq.~(\ref{local error feedback auxiliary update}), for momentum SGD with local error feedback we have
\begin{equation}\label{ef gradient mismatch}
\begin{split}
    &\tilde{\textbf{x}}_{t+1} = \tilde{\textbf{x}}_t - \frac{\eta_t}{K}\sum^{K}_{k=1}(\mu \textbf{m}^{(k)}_t+\nabla f(\textbf{x}^{(k)}_t;\xi^{(k)}_t))=\\
    &\tilde{\textbf{x}}_t - \frac{\eta_t}{K}\sum^{K}_{k=1}(\mu \textbf{m}^{(k)}_t+\nabla f(\tilde{\textbf{x}}_t+(\textbf{e}_t+\frac{1}{K}\sum^{K}_{k=1}\textbf{e}^{(k)}_t);\xi^{(k)}_t))\,.
\end{split}
\end{equation}

\textbf{Relationship with Staleness.} Asynchronous distributed training behaves in a similar pattern as the above equation. Let the staleness of the gradient computed at worker $k$ be $\tau^{(k)}_t$ and one worker is selected to update the model at the server node in each iteration. Then we have the following update rule in asynchronous SGD~\cite{lian2015asynchronous}:
\begin{equation}\label{asynchronous gradient mismatch}
    \textbf{x}_{t+1}=\textbf{x}_t-\eta_t\nabla f(\textbf{x}_{t-\tau^{(k)}_t};\xi^{(k)}_t)\,,
\end{equation}
where the gradient mismatch also exists as parameters $\textbf{x}_t$ is updated by the gradient computed at stale and different parameters $\textbf{x}_{t-\tau^{(k)}_t}$. Consequently we regard the staleness of local error feedback as one because the compression error $(\textbf{e}_t+\frac{1}{K}\sum^{K}_{k=1}\textbf{e}^{(k)}_t)$ in Eq.~(\ref{ef gradient mismatch}) is computed at iteration $t-1$, while $\tau^{(k)}_t$ in Eq.~(\ref{asynchronous gradient mismatch}) can be larger than one. Addressing gradient mismatch can be equivalent to reducing the effect of staleness.

The motivation for us to resolve gradient mismatch problem is that since $\textbf{x}_t$ differs from $\tilde{\textbf{x}}_t$ in the compression error term $(\textbf{e}_t+\frac{1}{K}\sum^{K}_{k=1}\textbf{e}^{(k)}_t)$, we can \textbf{improve the training of $\textbf{x}_t$ by improving the training of $\tilde{\textbf{x}}_t$}. To free the training of $\tilde{\textbf{x}}_t$ from the effect of gradient mismatch, we propose to approximate the following update rules:
\begin{equation}
    \tilde{\textbf{x}}_{t+1}\approx\tilde{\textbf{x}}_t - \frac{\eta_t}{K}\sum^{K}_{k=1}(\mu \textbf{m}^{(k)}_t+\nabla f(\tilde{\textbf{x}}_t;\xi^{(k)}_t))\,.
\end{equation}
Before to quantitatively define the amount of the gradient mismatch, we first made some common assumptions in non-convex optimization.
\begin{assumption}\label{lipschitz gradient}
($L$-Lipschitz gradient) Assume the full loss function $F(\cdot)$ is $L$-smooth, that is, $\forall \textbf{x}, \textbf{y}\in \mathbb{R}^d$ we have:
\begin{equation}
    \|\nabla F(\textbf{x})-\nabla F(\textbf{y})\|_2\leq L\|\textbf{x}-\textbf{y}\|_2\,.
\end{equation}
\end{assumption}
\begin{assumption}\label{bounded variance}
(Bounded variance) The stochastic gradient $\nabla f(\textbf{x}^{(k)}_t;\xi^{(k)}_t)$ has bounded variance:
\begin{equation}
    \mathbb{E}\|\nabla f(\textbf{x}^{(k)}_t;\xi^{(k)}_t)-\nabla F(\textbf{x}^{(k)}_t)\|^2_2\leq \sigma^2\,.
\end{equation}
\end{assumption}
With the Assumptions \ref{lipschitz gradient} and \ref{bounded variance}, we define the amount of the gradient mismatch $\epsilon_t$ of local error feedback as:
\begin{equation}
\begin{split}
    \epsilon_t &\coloneqq\frac{1}{K}\sum^{K}_{k=1}\mathbb{E}\|\nabla f(\tilde{\textbf{x}}_t;\xi^{(k)}_t)-\nabla f(\textbf{x}^{(k)}_t;\xi^{(k)}_t)\|^2_2\\
    &\leq L^2\mathbb{E}\|\textbf{e}_t+\frac{1}{K}\sum^{K}_{k=1}\textbf{e}^{(k)}_t\|^2_2 +4\sigma^2\,.
\end{split}
\end{equation}

\subsection{Step Ahead}

Although local error feedback is proved to have the same convergence rate $\mathcal{O}(\frac{1}{\sqrt{T}})$ as SGD, the gradient mismatch $\epsilon_t$ leads to an additional error term in the convergence bound. In stead of computing stochastic gradient $\nabla f(\textbf{x}^{(k)}_t;\xi^{(k)}_t)$ at $\textbf{x}^{(k)}_t$, we propose to compute stochastic gradient $\nabla f(\textbf{x}^{(k)}_{t+\frac{1}{2}};\xi^{(k)}_t)$ at $\textbf{x}^{(k)}_{t+\frac{1}{2}}\coloneqq \textbf{x}^{(k)}_t-\textbf{e}^{(k)}_t$ as in Algorithm \ref{error feedback algorithm} (line 6 $\sim$ 12). Note that the local error $\textbf{e}^{(k)}_t$ is locally accessible without additional communication costs. By replacing $\textbf{x}^{(k)}_t$ with $\textbf{x}^{(k)}_{t+\frac{1}{2}}$, the gradient mismatch $\epsilon_t$ of our proposed SAEF becomes
\begin{equation}
\begin{split}
    \epsilon_t &\coloneqq\frac{1}{K}\sum^{K}_{k=1}\mathbb{E}\|\nabla f(\tilde{\textbf{x}}_t;\xi^{(k)}_t)-\nabla f(\textbf{x}^{(k)}_{t+\frac{1}{2}};\xi^{(k)}_t)\|_2^2\\
    &\leq \frac{L^2}{K}\sum^{K}_{k=1}\mathbb{E}\|\textbf{e}_t+\frac{1}{K}\sum^{K}_{k=1}\textbf{e}^{(k)}_t-\textbf{e}^{(k)}_t\|_2^2 +4\sigma^2\,.
\end{split}
\end{equation}

\textbf{When to step ahead?} Our goal is to fix gradient mismatch (reduce staleness) to improve the training of $\tilde{\textbf{x}}_t$. As we want a smaller upper bound of $\epsilon_t$, it will be better to step ahead if $\frac{1}{K}\sum^{K}_{k=1}\mathbb{E}\|\textbf{e}_t+\frac{1}{K}\sum^{K}_{k=1}\textbf{e}^{(k)}_t-\textbf{e}^{(k)}_t\|^2_2 < \mathbb{E}\|\textbf{e}_t+\frac{1}{K}\sum^{K}_{k=1}\textbf{e}^{(k)}_t\|^2_2$. This is intuitively true when we have a small variance because the effect of $(\frac{1}{K}\sum^{K}_{k=1}\textbf{e}^{(k)}_t-\textbf{e}^{(k)}_t)$ is cancelled in expectation. The following proposition illustrates it when the variance is smaller than the square of expectation. Note it only gives a motivation of our proposed method as the local error is usually not identical in SAEF-SGD and EF-SGD.

\begin{proposition}
If Assumptions \ref{lipschitz gradient} and \ref{bounded variance} exist, for the the same error $\textbf{e}^{(k)}_t$ ($k=1,\cdots,K$) and $\textbf{e}_t$, the upper bound of $\epsilon_t$ we can prove in SAEF-SGD is better than that in EF-SGD if $\text{Var}(\textbf{e}^{(k)}_t)\leq \|\mathbb{E}\textbf{e}^{(k)}_t\|^2_2$.
\end{proposition}

\textbf{How related to compression?} Take the flexible Top-K gradient compression as an example, where only large gradient components are sent with the rest set to zero. We regard $\epsilon_t=\|(\epsilon_{t,1},\cdots,\epsilon_{t,d})\|^2_2$ in an element-wise way, which means that the improvement of $\epsilon_t$ related to $\textbf{e}^{(k)}_t$ of SAEF-SGD over EF-SGD is proportional to the number of non-zero components in $\textbf{e}^{(k)}_t$. When the Top-K compression is more aggressive and fewer gradient components are sent, there are more non-zero components in $\textbf{e}^{(k)}_t$. In other words, the improvement of SAEF-SGD over EF-SGD favors more aggressive Top-K compression, which is desirable due to lower communication costs. The less aggressive compression incurs smaller performance loss, but the improvement of local error feedback is not as essential.

\subsection{Error Averaging}

When the gradient mismatch is too hard to resolve only by step ahead, we propose to average the local error $\textbf{e}^{(k)}_t\leftarrow \frac{1}{K}\sum^{K}_{k=1}\textbf{e}^{(k)}_t$. It cancels the effect of local compression error $\textbf{e}^{(k)}_t$ with a brutal force at the averaging iteration $t$:
\begin{equation}
\begin{split}
    \epsilon_t &\coloneqq\frac{1}{K}\sum^{K}_{k=1}\mathbb{E}\|\nabla f(\tilde{\textbf{x}}_t;\xi^{(k)}_t)-\nabla f(\textbf{x}^{(k)}_{t+\frac{1}{2}};\xi^{(k)}_t)\|^2_2\\
    &\leq \frac{L^2}{K}\sum^{K}_{k=1}\mathbb{E}\|\textbf{e}_t\|^2_2+4\sigma^2\,.
\end{split}
\end{equation}

The error averaging operation can be conducted either in a master-slave way or the ring-based all-reduce way to avoid the traffic jam in the master-slave framework. However, this is still a costly operation and we do not want to conduct it frequently. In fact, we should average the local error every $p (>1)$ iteration depending on how fast the local error diverges in different nodes. When $p=\infty$ we do not perform error averaging. To make a fair comparison in experiments, for SAEF with error averaging we apply the less aggressive gradient compression to balance the communication cost.

\textbf{How much contribution?} Error averaging set $\mathbb{E}\|\frac{1}{K}\sum^{K}_{k=1}\textbf{e}^{(k)}_t-\textbf{e}^{(k)}_t\|_2^2$ to zero every $p$ iteration. It reduces the upper bound of $\epsilon_t$ related to $\textbf{e}^{(k)}_t$ at least by a factor of $\frac{1}{p}$. Moreover, it prevents the local error $\textbf{e}^{(k)}_t$ in different worker $k$ from further diverging. Consequently averaging error every $p$ iteration reduce the effect of local error $\textbf{e}^{(k)}_t$ at worker nodes by a factor larger than $\frac{1}{p}$.

\section{Theoretical Analysis}\label{theoretical results}
We further make Assumption \ref{bounded second moment} which is common in non-convex optimization, and Assumption \ref{compressor} which has been leveraged in previous works~\cite{stich2018sparsified,karimireddy2019error,zheng2019communication,basu2019qsparse}. For simplicity, we denote $\min_{t=0,1,\cdots,T-1}$ as $\min$. The theoretical results do not include error averaging.
\begin{assumption}\label{bounded second moment}
(Bounded second moment) The full gradient is bounded:
\begin{equation}
    \|\nabla F(\textbf{x}^{(k)}_t)\|^2_2\leq M^2\,.
\end{equation}
It implies the second moment of the stochastic gradient is bounded if Assumption \ref{bounded variance} exists at the same time:
\begin{equation}
    \mathbb{E}\|\nabla f(\textbf{x}^{(k)}_t;\xi^{(k)}_t)\|_2^2\leq \sigma^2+M^2\,.
\end{equation}
\end{assumption}
\begin{assumption}\label{compressor}
($\delta$-approximate compressor) The compression function $\mathcal{C}(\cdot)$ : $\mathbb{R}^d\to\mathbb{R^d}$ is a $\delta$-approximate compressor for $0<\delta\leq 1$ if for all $\textbf{v}\in\mathbb{R}^d$,
\begin{equation}
    \|\mathcal{C}(\textbf{v})-\textbf{v}\|_2^2\leq (1-\delta)\|\textbf{v}\|_2^2\,.
\end{equation}
\end{assumption}
\begin{lemma}\label{mismatch bound}
With Assumptions \ref{bounded variance}, \ref{bounded second moment} and \ref{compressor}, we have
\begin{equation}
\begin{split}
    &\frac{1}{K}\sum^{K}_{k=1}\mathbb{E}\|\textbf{e}_t+\frac{1}{K}\sum^{K}_{k=1}\textbf{e}^{(k)}_t-\textbf{e}^{(k)}_t\|_2^2\\
    &\leq C\cdot\frac{1-\delta}{(1-\sqrt{1-\delta)^2}}\frac{\eta_{max}^2(M^2+\sigma^2)}{(1-\mu)^2}\,,
\end{split}
\end{equation}
where the constant $C=\frac{2(1+\delta)(2-\delta)}{(1-\sqrt{1-\delta})^2}+\frac{1+\delta}{\delta}$.
\end{lemma}

Lemma \ref{mismatch bound} is an essential intermediate result for the convergence analysis of SAEF both with or without momentum as it helps to bound the gradient mismatch $\epsilon_t$.

\begin{figure*}[t]
\centering
    \includegraphics[width=0.245\textwidth]{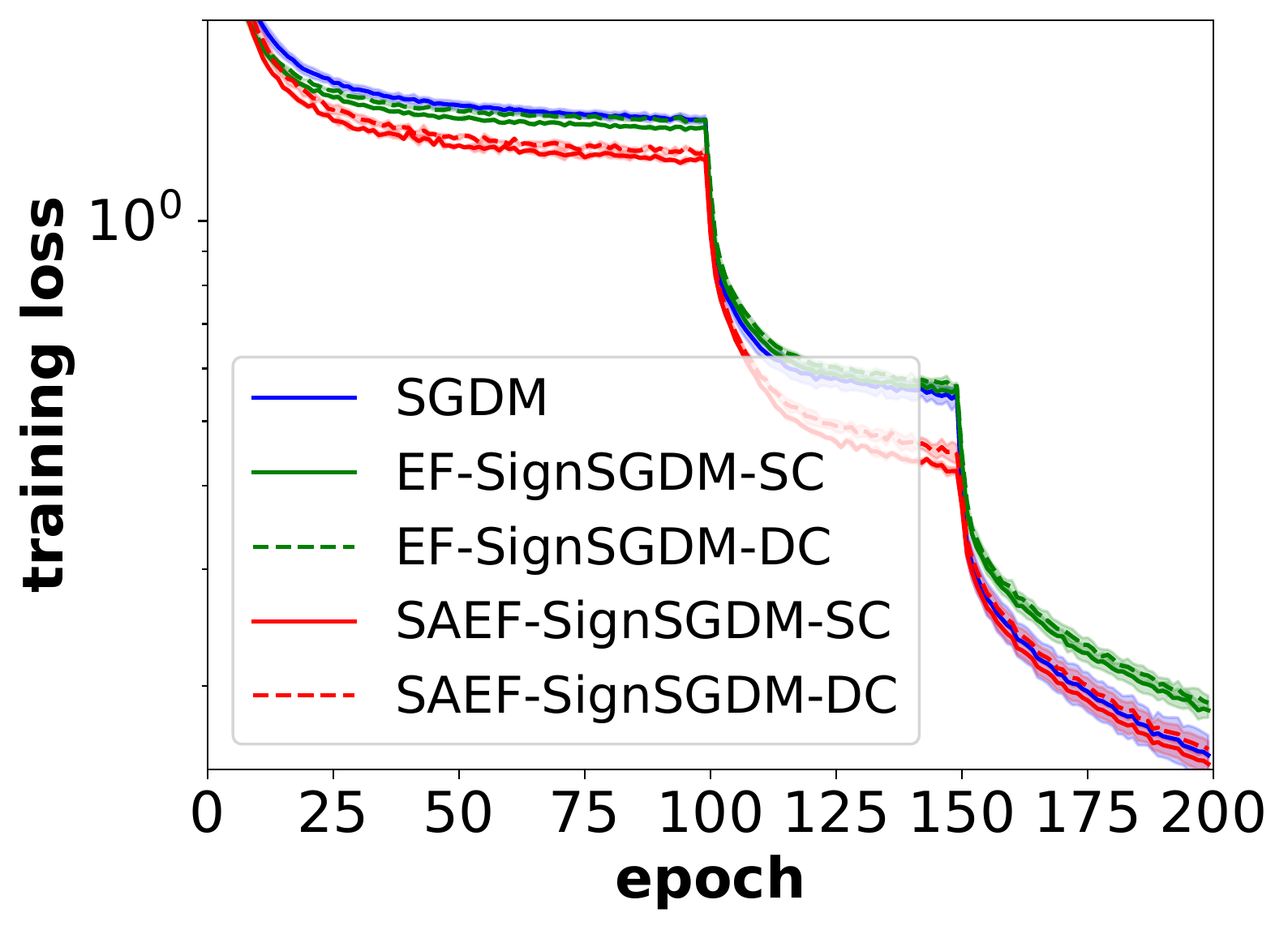}
    \includegraphics[width=0.245\textwidth]{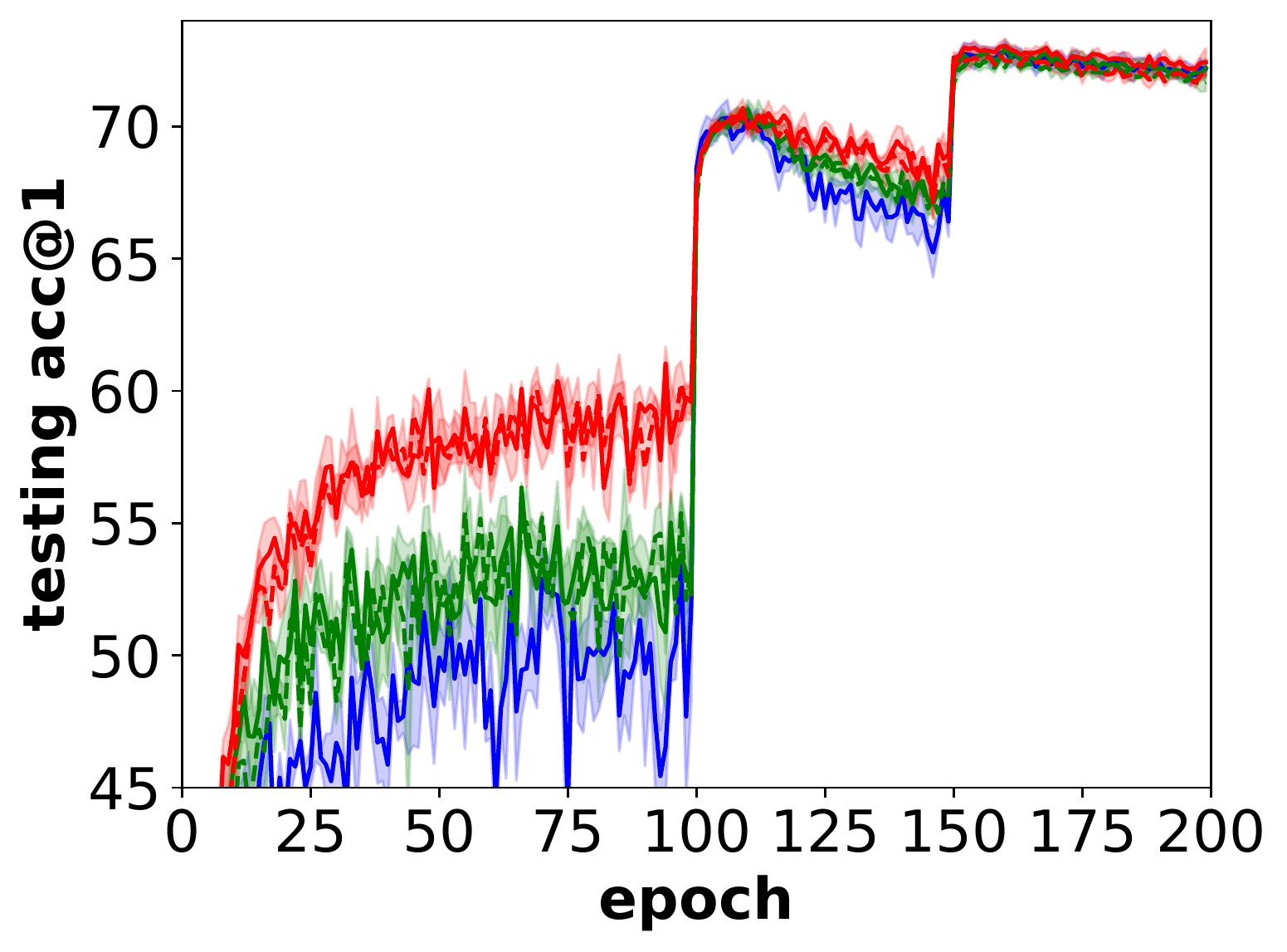}
    \includegraphics[width=0.245\textwidth]{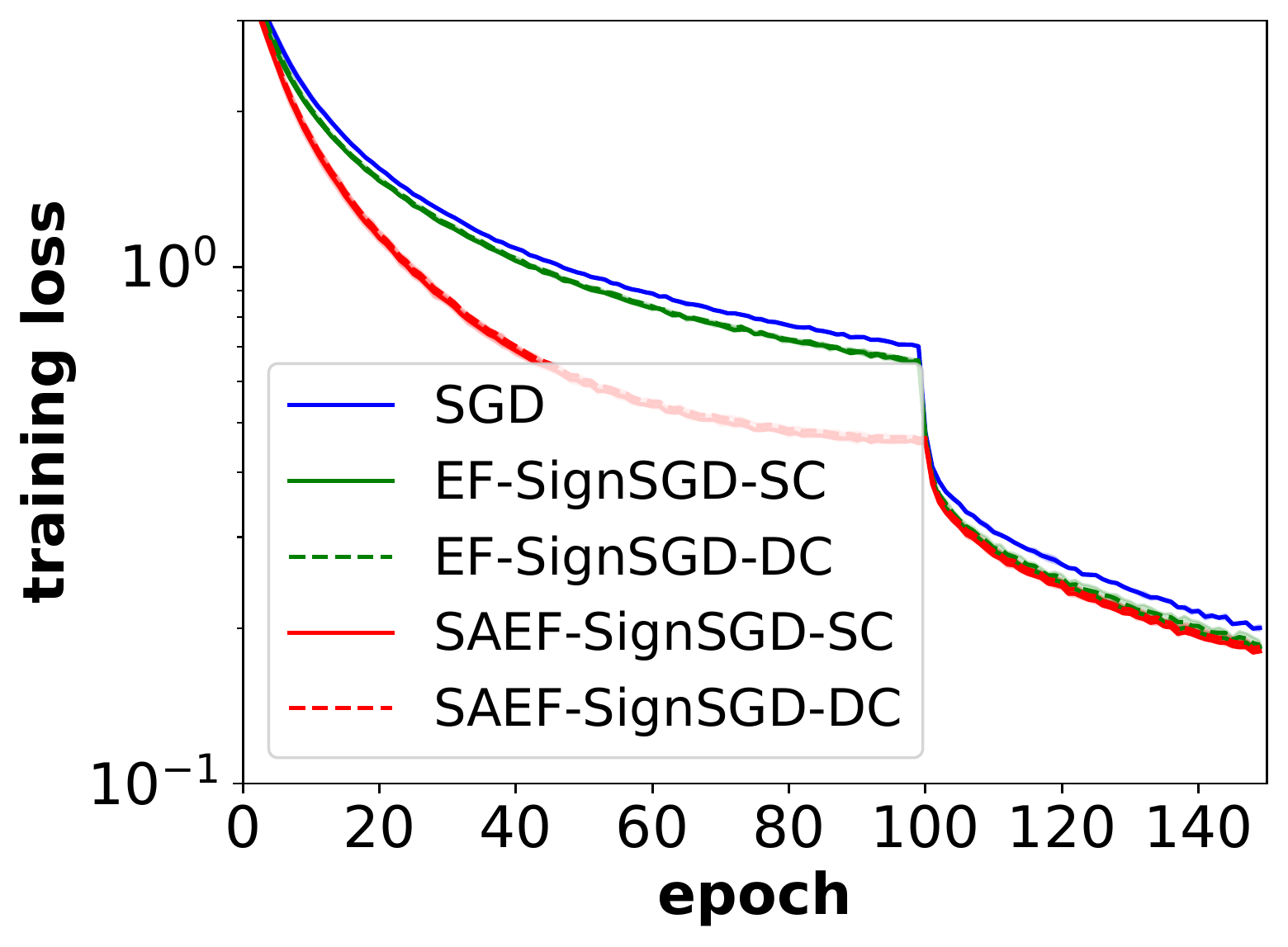}
    \includegraphics[width=0.245\textwidth]{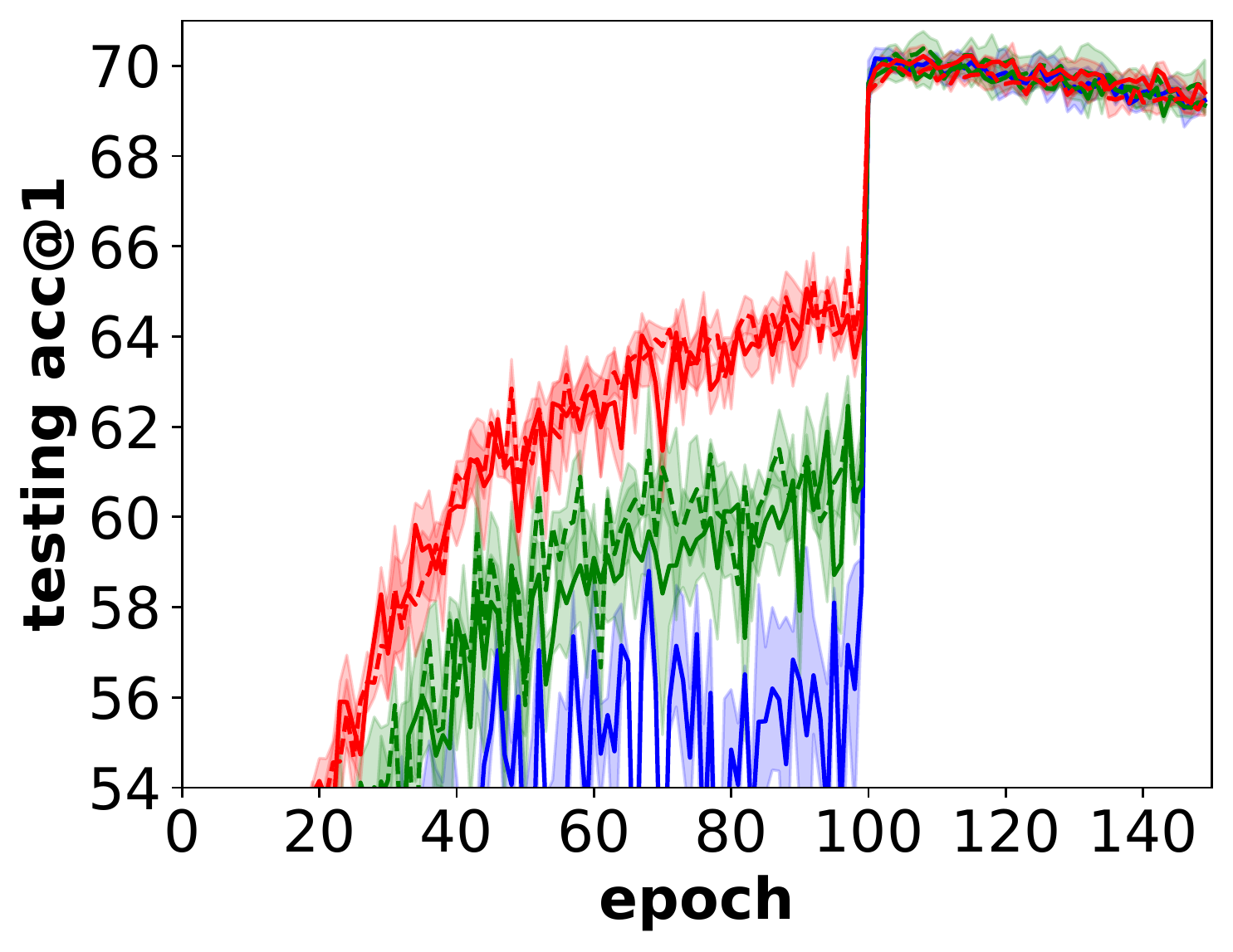}
    \includegraphics[width=0.245\textwidth]{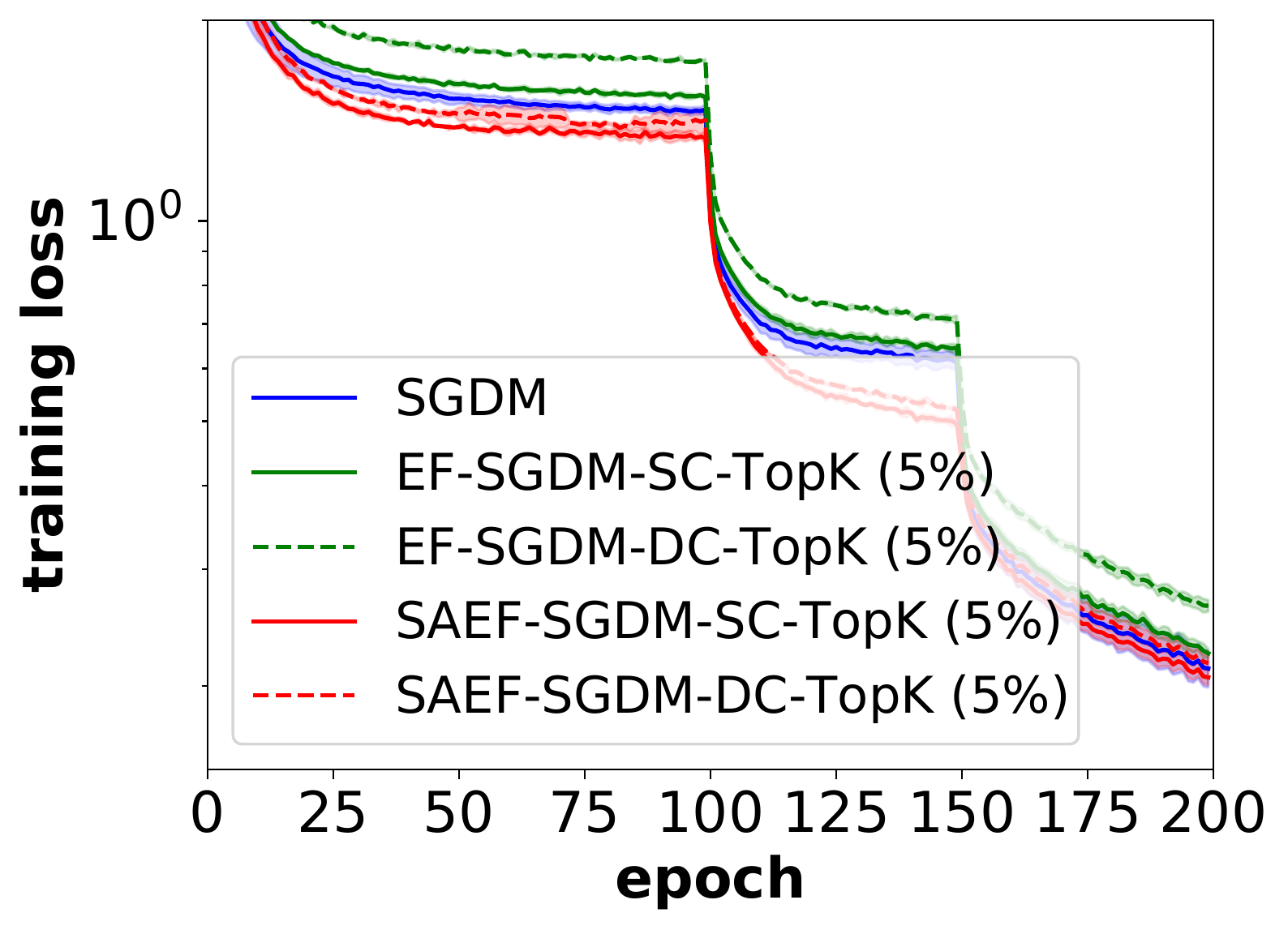}
    \includegraphics[width=0.245\textwidth]{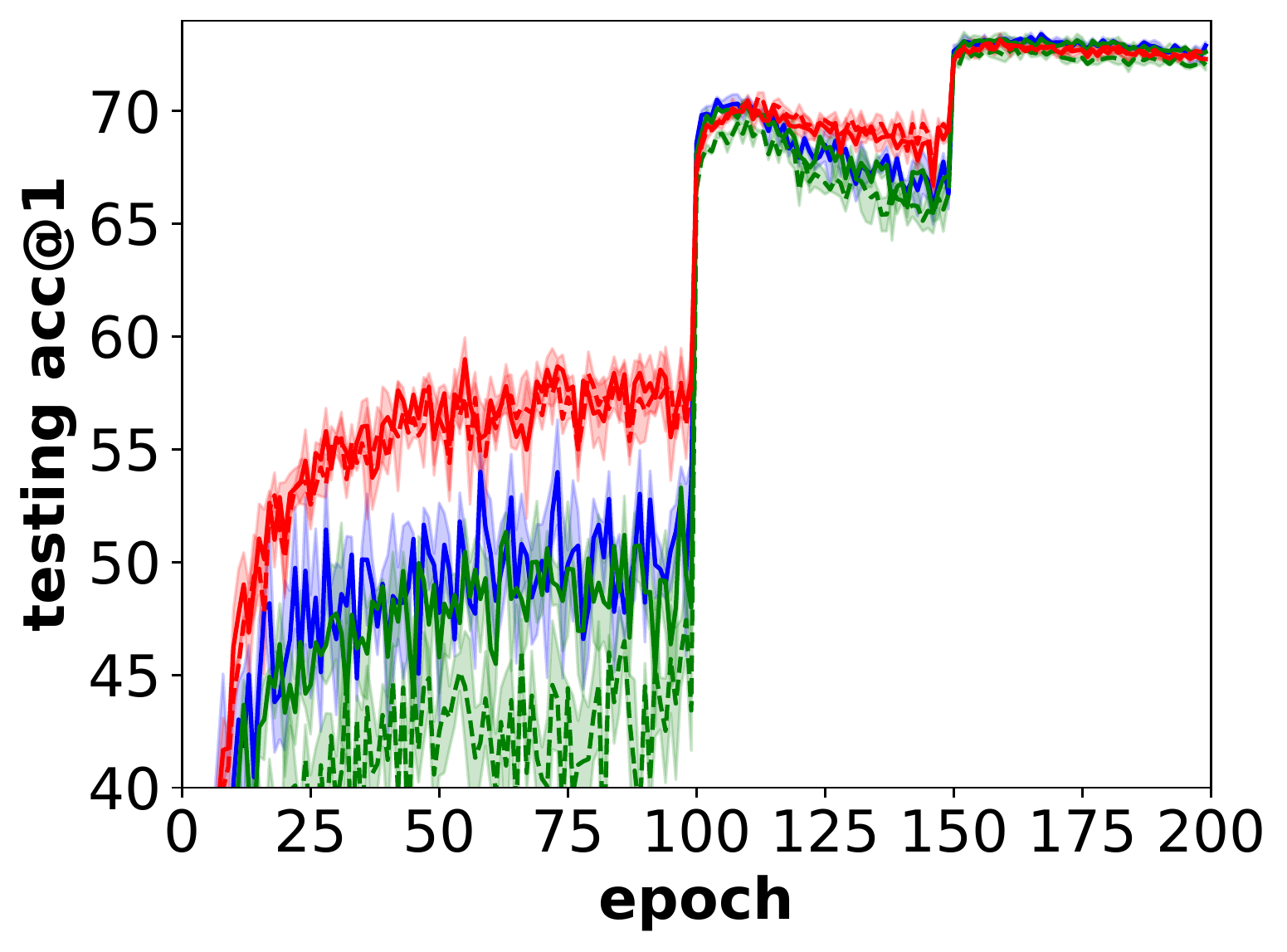}
    \includegraphics[width=0.245\textwidth]{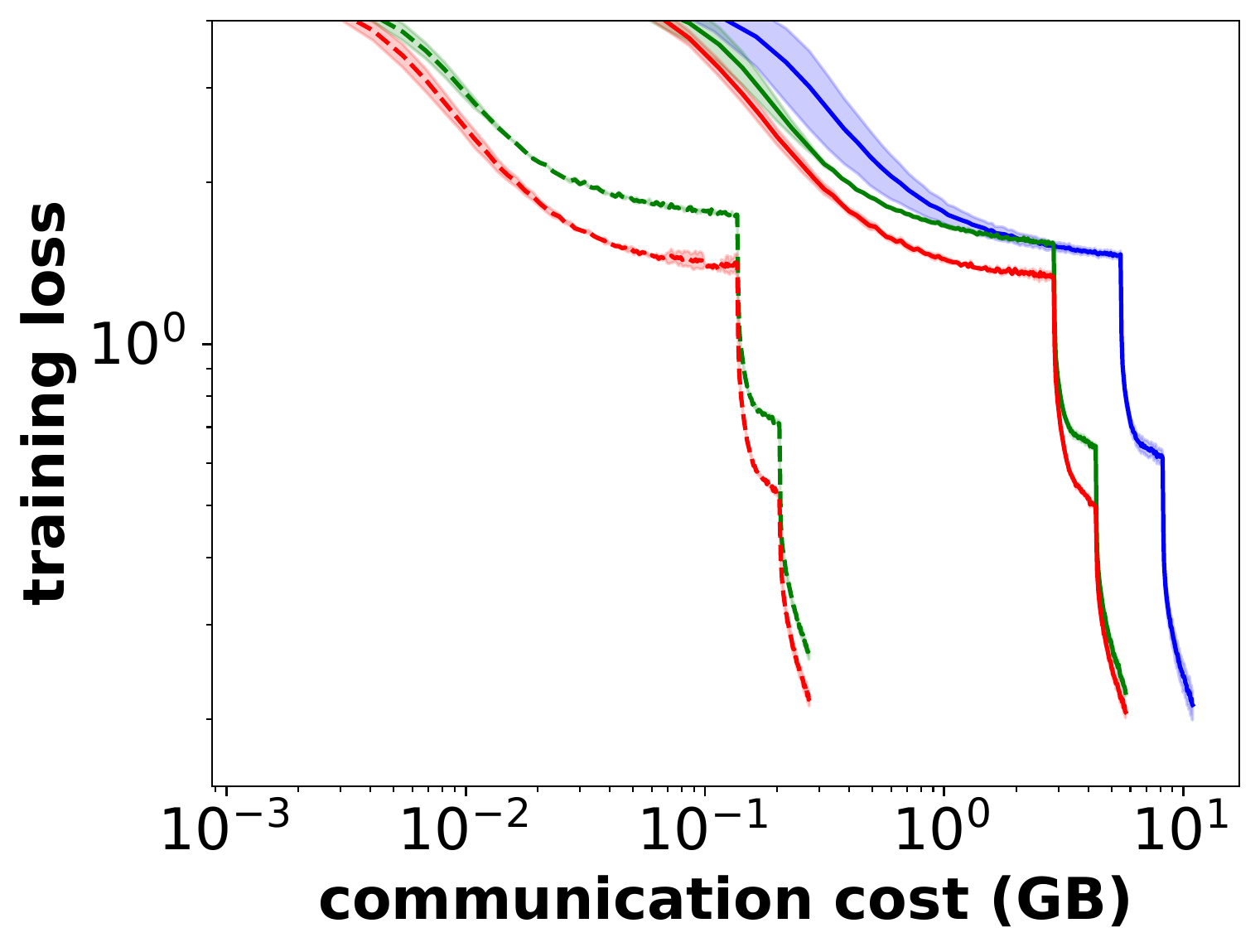}
    \includegraphics[width=0.245\textwidth]{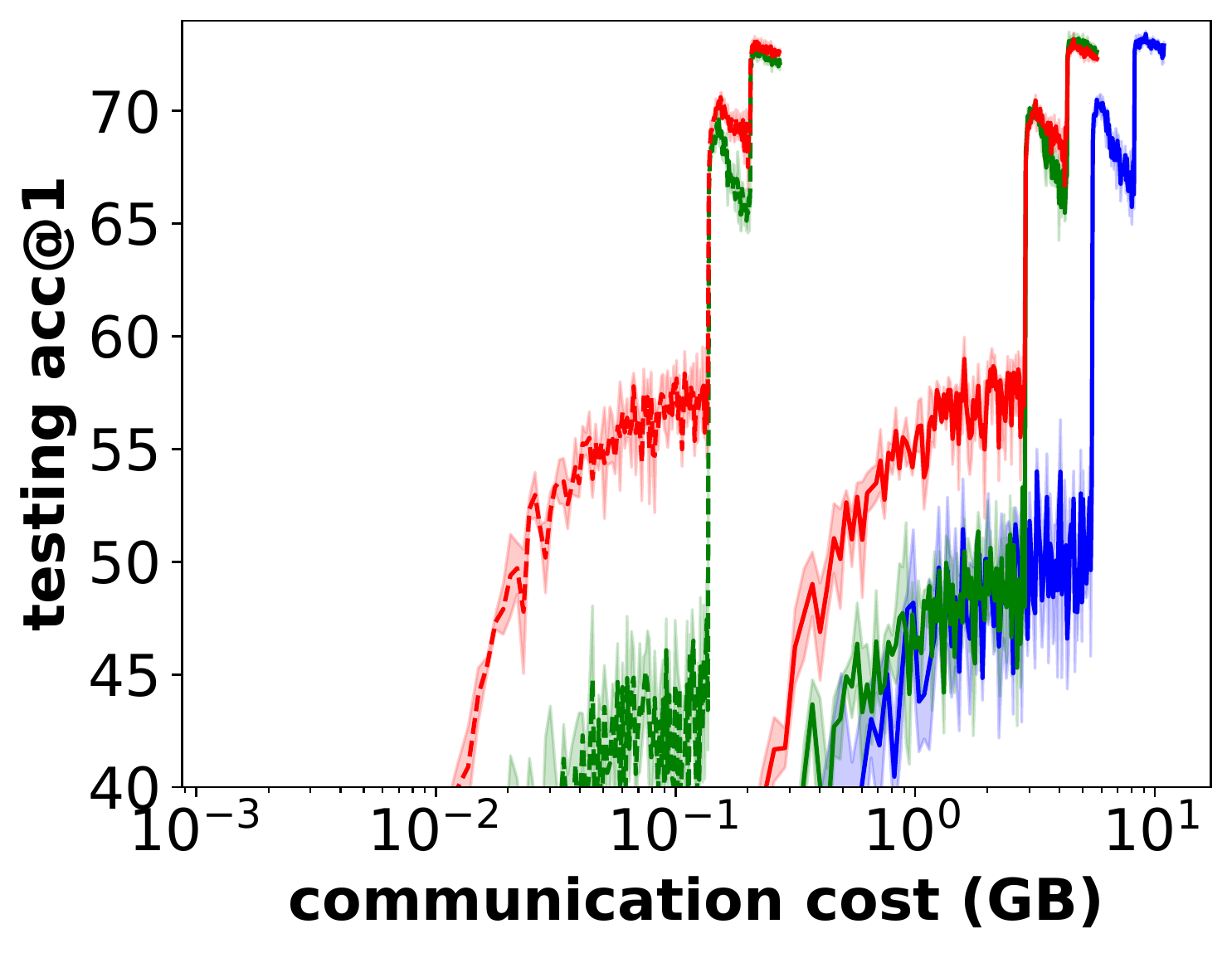}
\caption{Train ResNet-56 on CIFAR-100. Mean metrics are plotted with standard deviation (shaded area). The top row employs 4 workers and SignSGD compression with momentum SGD applied in the left two figures and SGD applied in the right two figures. The bottom row employs 8 workers, Top-K compression and momentum SGD, where training curves regarding epochs are shown in the left two figures and training curves regarding communication costs are shown in the right two figures.}
\label{cifar100 training curves.}
\end{figure*}

\subsection{SAEF-SGD}

\begin{theorem}\label{saef-sgd auxiliary convergence}
If Assumptions \ref{lipschitz gradient}, \ref{bounded variance}, \ref{bounded second moment} and \ref{compressor} exist, and the learning rate $0<\eta_t=\eta<\frac{3}{4L}$ for all $t=0,\cdots,T-1$, for SAEF-SGD we have
\begin{equation}
\begin{split}
    &\min\mathbb{E}\|\nabla F(\tilde{\textbf{x}}_t)\|_2^2 \leq \frac{4[F(\tilde{\textbf{x}}_0)-F(\tilde{\textbf{x}}^*)]}{\eta(3-4\eta L)T} + \frac{2\eta L\sigma^2}{(3-4\eta L)K}\\
    &+ \frac{C(1-\delta)}{(1-\sqrt{1-\delta})^2} \frac{4(\eta L+1)\eta^2L^2}{3-4\eta L}(M^2+\sigma^2)\,.
\end{split}
\end{equation}
\end{theorem}

\begin{theorem}\label{saef-sgd convergence}
If Assumptions \ref{lipschitz gradient}, \ref{bounded variance}, \ref{bounded second moment} and \ref{compressor} exist, and the learning rate $0<\eta_t=\eta<\frac{3}{2L}$ for all $t=0,\cdots,T-1$, for SAEF-SGD we have
\begin{equation}
\begin{split}
    &\min\mathbb{E}\|\nabla F(\textbf{x}^{(k)}_t)\|_2^2 \leq \frac{4[F(\textbf{x}_0)-F(\textbf{x}^*)]}{\eta(3-2\eta L)T} +\frac{4\eta L\sigma^2}{(3-2\eta L)K}\\
    &+\left(2+\frac{8C}{3-2\eta L}\right)\frac{1-\delta}{(1-\sqrt{1-\delta})^2}\eta^2L^2(M^2+\sigma^2)\,.
\end{split}
\end{equation}
\end{theorem}

\begin{corollary}
Under the same conditions of Theorem \ref{saef-sgd convergence}, the compression error term
\begin{equation}
    \left(2+\frac{8C}{3-2\eta L}\right) \frac{1-\delta}{(1-\sqrt{1-\delta})^2}\eta^2L^2 (M^2+\sigma^2)
\end{equation}
in the upper bound of Theorem \ref{saef-sgd convergence} is much tighter than the corresponding EF-SGD compression error term in \cite{zheng2019communication}:
\begin{equation}
    \frac{32L^2(1-\delta)(M^2+\sigma^2)}{\delta^2}(1+\frac{16}{\delta^2})\frac{\eta^2}{3-2\eta L}\,.
\end{equation}
\end{corollary}

The effect of a tighter bound achieved above by SAEF will gradually vanish with a decaying learning rate as the total training steps $T$ goes to infinity. However, in practical and common training of deep neural networks, the learning rate is usually chosen to be large in the beginning and seldom goes to zero in the end, which contributes to the faster training of SAEF than local error feedback.

\begin{corollary}\label{saef-sgd convergence rate}
Under the same conditions of Theorem \ref{saef-sgd convergence}, let the learning rate $\eta<\frac{c\sqrt{K}}{\sqrt{T}}$, where $c>0$ is some constant. Then the convergence rate of $\textbf{x}^{(k)}_t$ in SAEF-SGD satisfies
\begin{equation}
    \min_{t=0,\cdots,T-1}\mathbb{E}\|\nabla F(\textbf{x}^{(k)}_t)\|_2^2 =\mathcal{O}(\frac{1}{\sqrt{KT}})\,.
\end{equation}
\end{corollary}

Please see Section 5 of the Supplement for the proof.

\subsection{SAEF-SGD with Momentum}

\begin{theorem}\label{saef-sgdm convergence}
If Assumption \ref{lipschitz gradient}, \ref{bounded variance}, \ref{bounded second moment} and \ref{compressor} exist, and the learning rate $0<\eta_t=\eta$ satisfies $\alpha\coloneqq 1-\frac{\eta L}{1-\mu}-\frac{2\mu^2\eta^2L^2}{(1-\mu)^4}>0$ for all $t=0,\cdots,T-1$, for SAEF-SGD with momentum we have
\begin{equation}
\begin{split}
    &\min\mathbb{E}\|\nabla F(\textbf{x}^{(k)}_t)\|_2^2 \leq \frac{4(1-\mu)[F(\textbf{x}_0)-F(\textbf{x}^*)]}{\alpha\eta T} \\
    &\quad+ \frac{2\left(1+\frac{2\mu^2\eta L}{(1-\mu)^3}\right)\eta L\sigma^2}{\alpha(1-\mu)K}\\
    &\quad+ \left(4C+2\alpha\right)\frac{1-\delta}{(1-\sqrt{1-\delta})^2}\frac{\eta^2L^2(M^2+\sigma^2)}{\alpha(1-\mu)^2}\,,
\end{split}
\end{equation}
\begin{equation}
\begin{split}
    &\min\mathbb{E}\|\nabla F(\tilde{\textbf{x}}_t)\|_2^2 \leq \frac{4(1-\mu)[F(\textbf{x}_0)-F(\textbf{x}^*)]}{\alpha\eta T}\\
    &\quad+ \frac{2\left(1+\frac{2\mu^2\eta L}{(1-\mu)^3}\right)\eta L\sigma^2}{\alpha(1-\mu)K}\\
    &\quad+(\frac{4}{\alpha}+2)\frac{C(1-\delta)}{(1-\sqrt{1-\delta})^2}\frac{\eta^2L^2(M^2+\sigma^2)}{(1-\mu)^2}\,.
\end{split}
\end{equation}
\end{theorem}
\begin{corollary}\label{saef-sgdm convergence rate}
Under the same conditions of Theorem \ref{saef-sgdm convergence}, let the learning rate $\eta<\frac{c\sqrt{K}}{\sqrt{T}}$, where $c>0$ is some constant. Then the convergence rate of $\textbf{x}^{(k)}_t$ in SAEF-SGD with momentum satisfies
\begin{equation}
    \min_{t=0,\cdots,T-1}\mathbb{E}\|\nabla F(\textbf{x}^{(k)}_t)\|_2^2 =\mathcal{O}(\frac{1}{\sqrt{KT}})\,.
\end{equation}
\end{corollary}

Please see Section 6 of the Supplement for the proof.

\section{Experiments}\label{experiments}

\begin{figure*}[t]
\centering
    \includegraphics[width=0.245\textwidth]{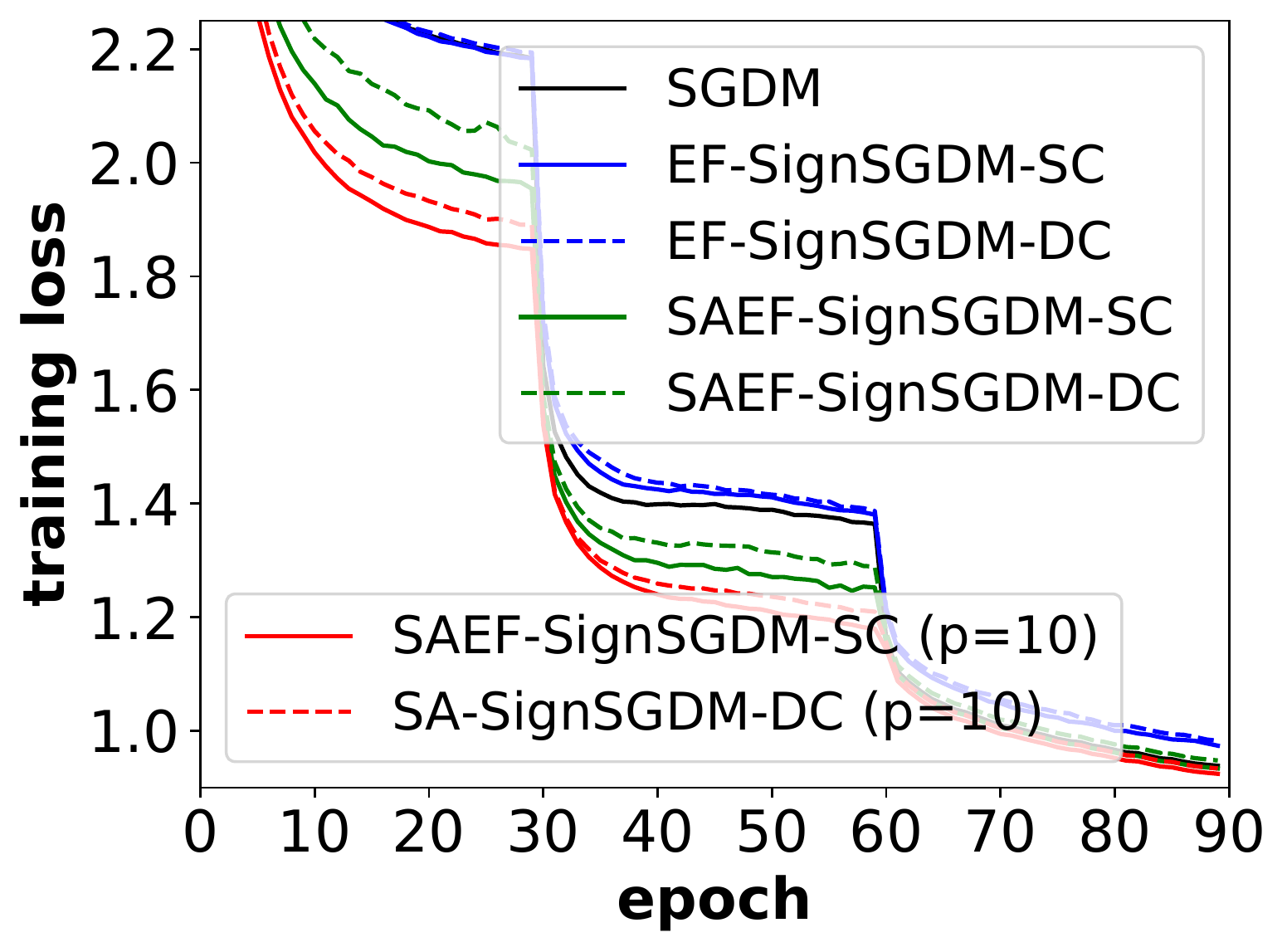}
    \includegraphics[width=0.245\textwidth]{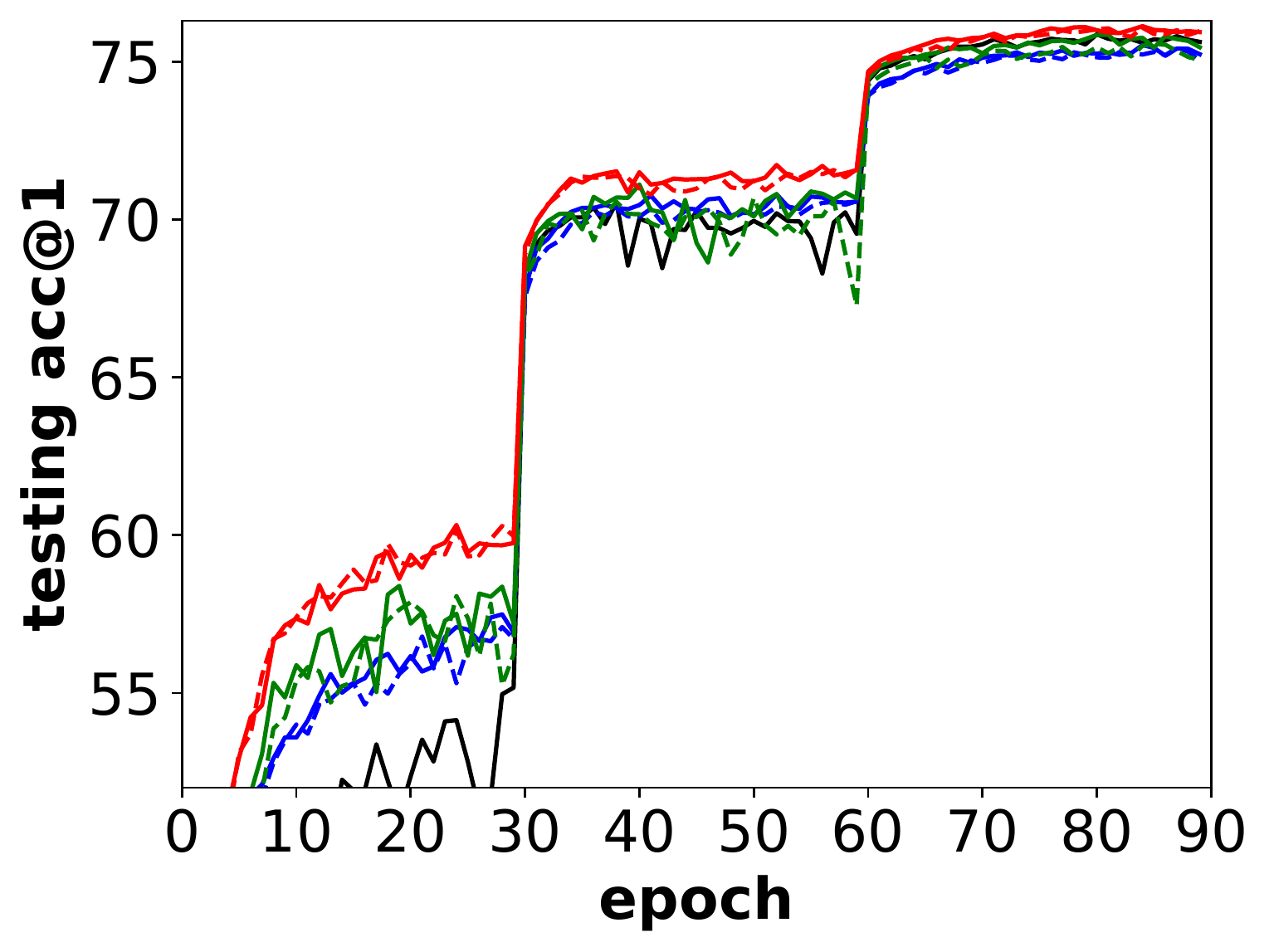}
    \includegraphics[width=0.245\textwidth]{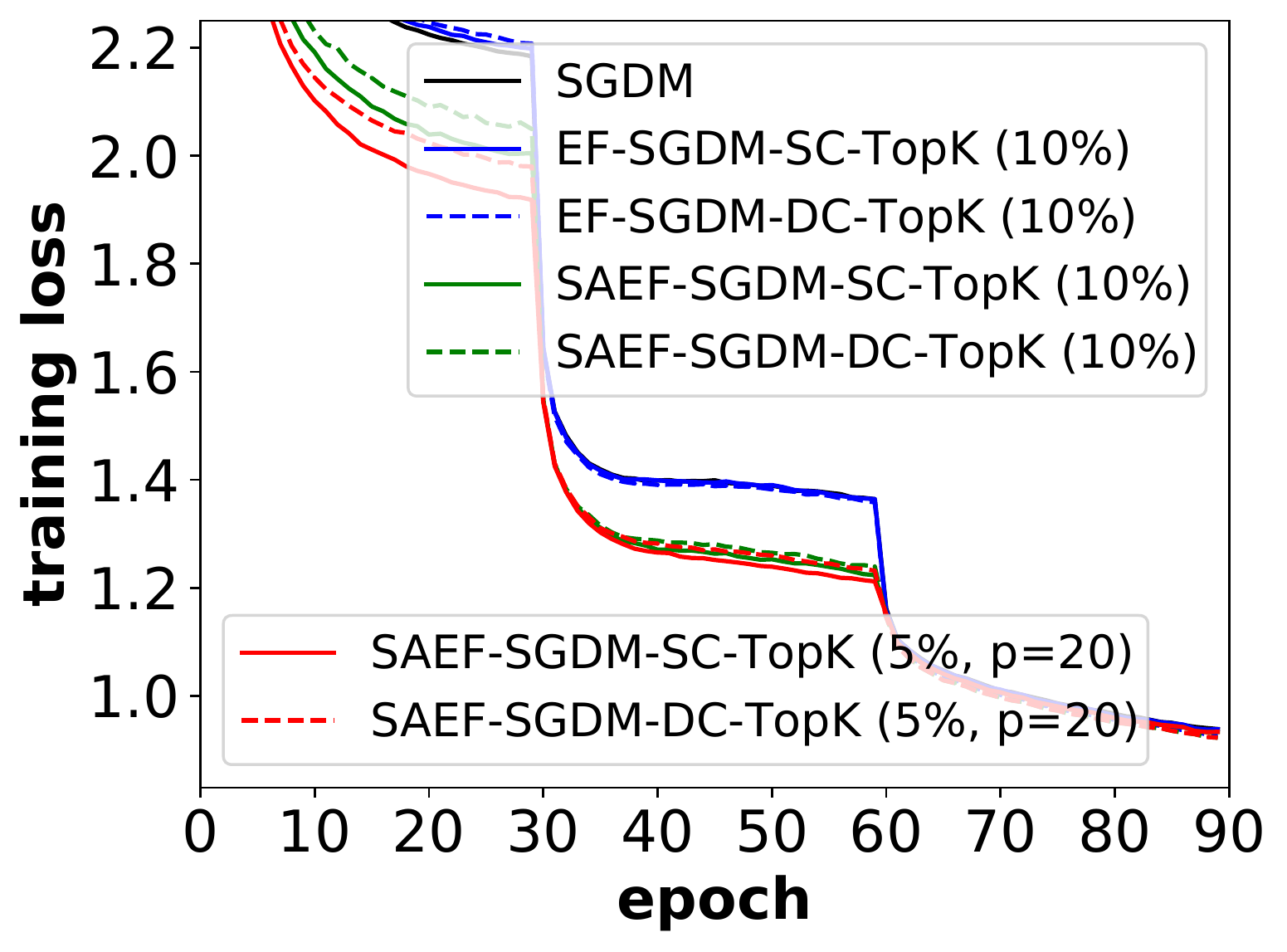}
    \includegraphics[width=0.245\textwidth]{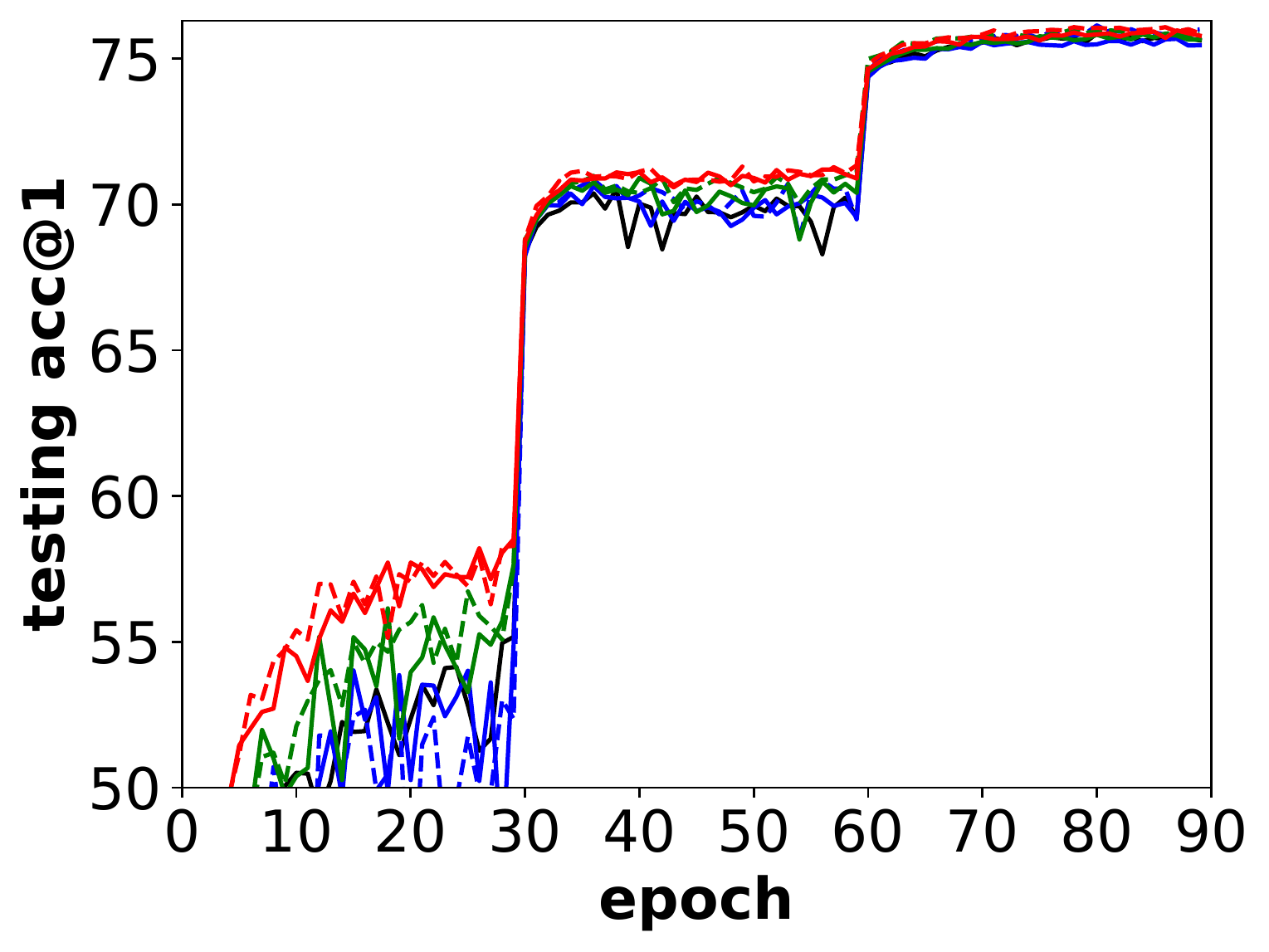}
    \includegraphics[width=0.245\textwidth]{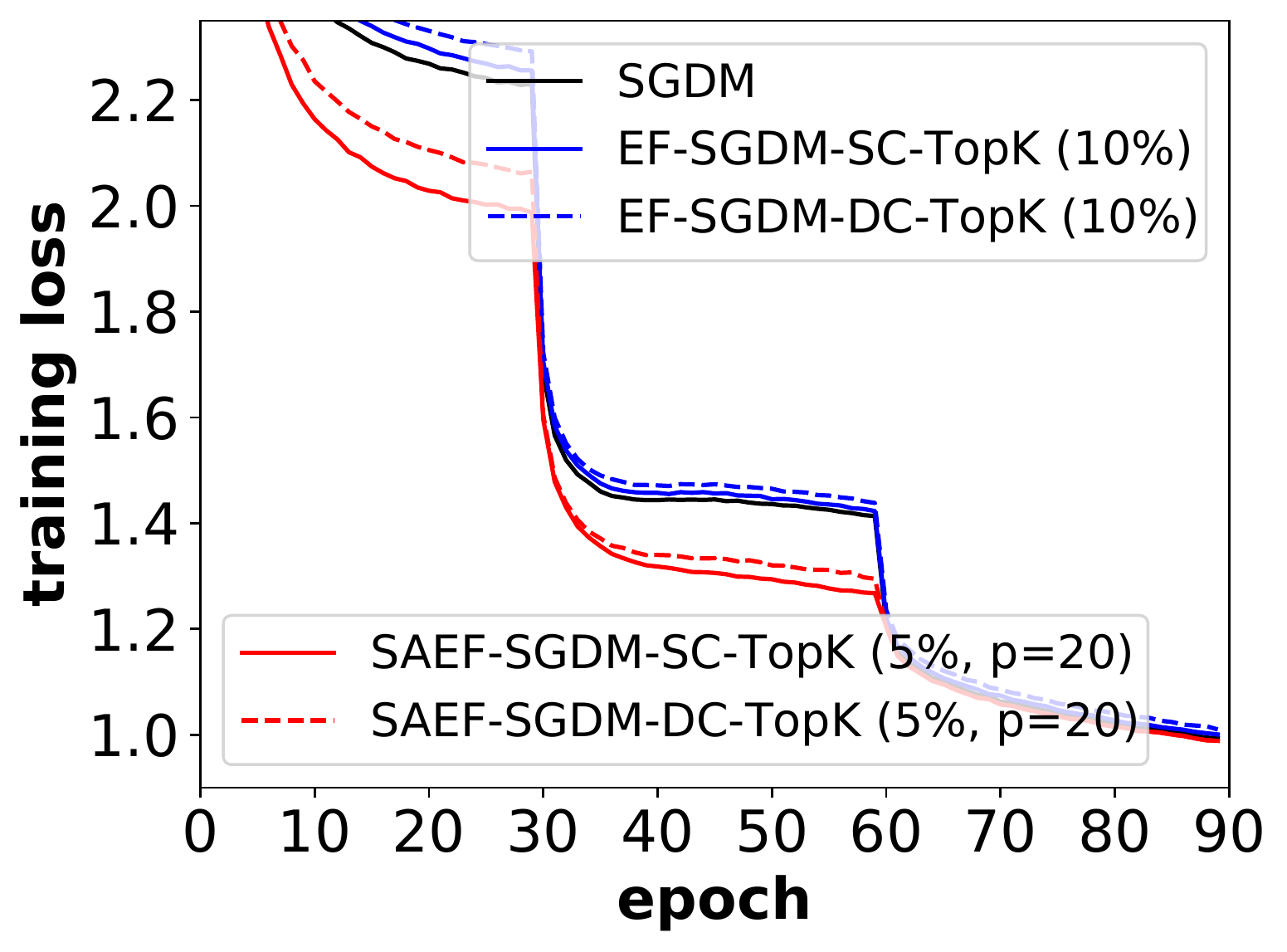}
    \includegraphics[width=0.245\textwidth]{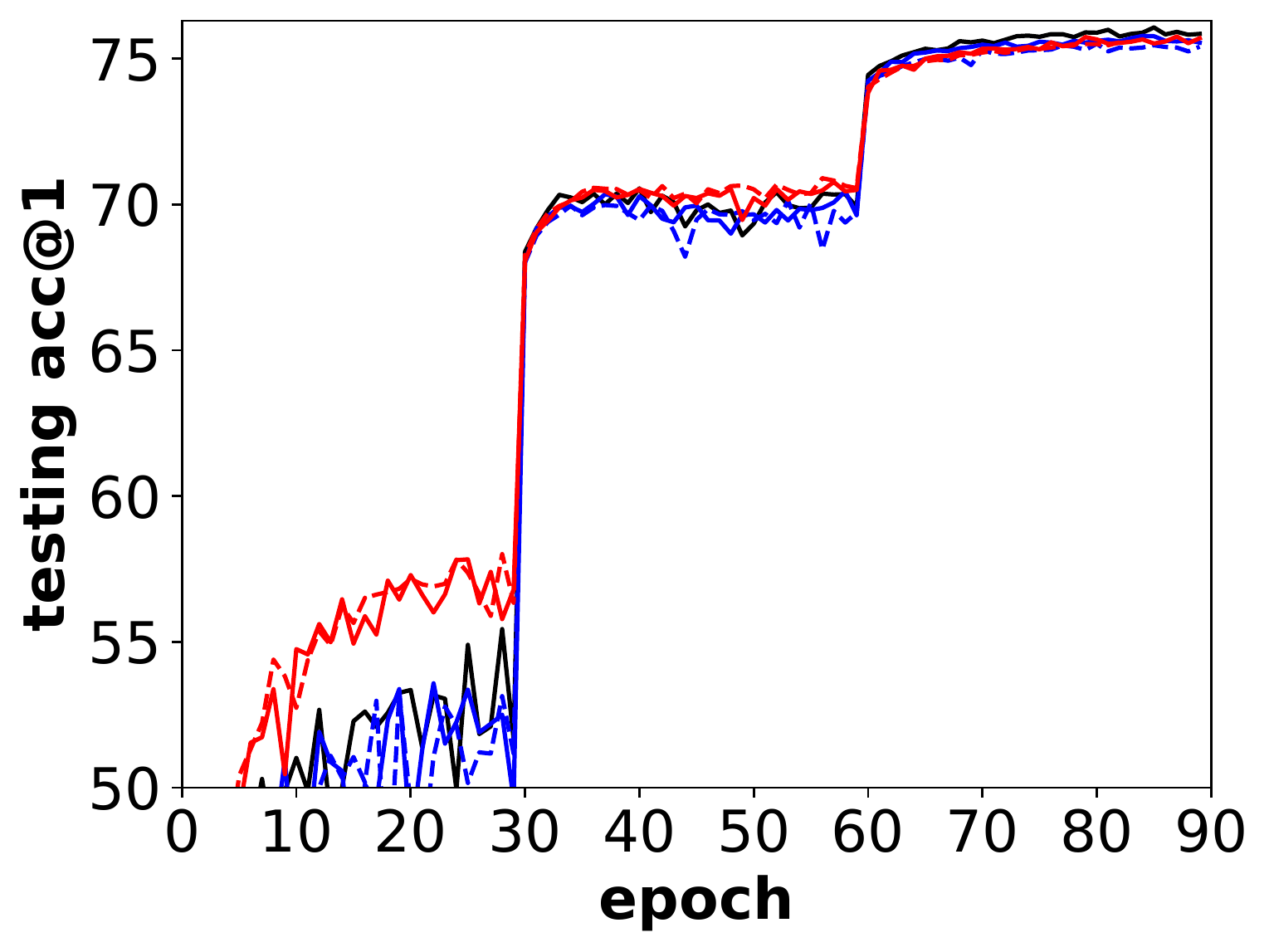}
    \includegraphics[width=0.245\textwidth]{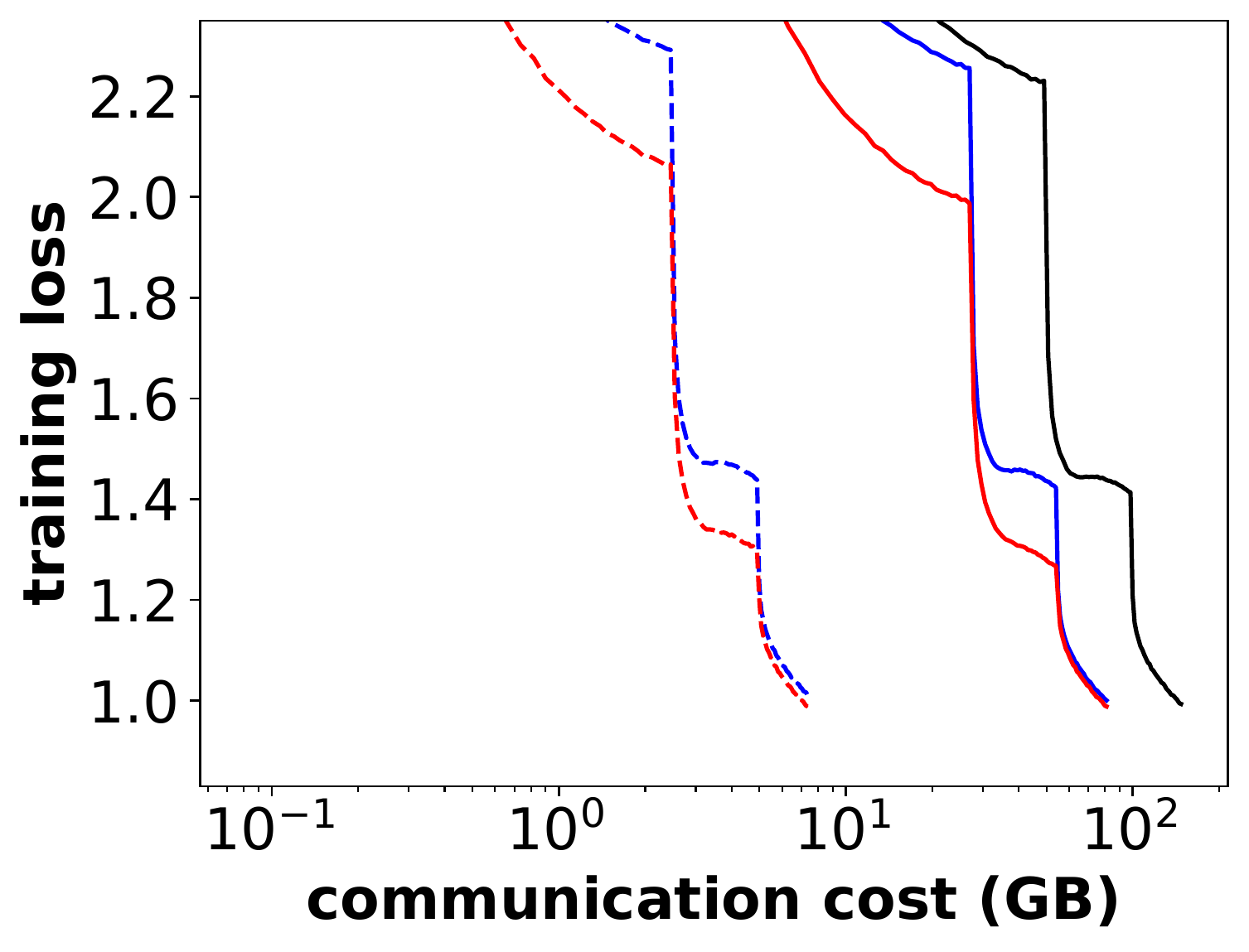}
    \includegraphics[width=0.245\textwidth]{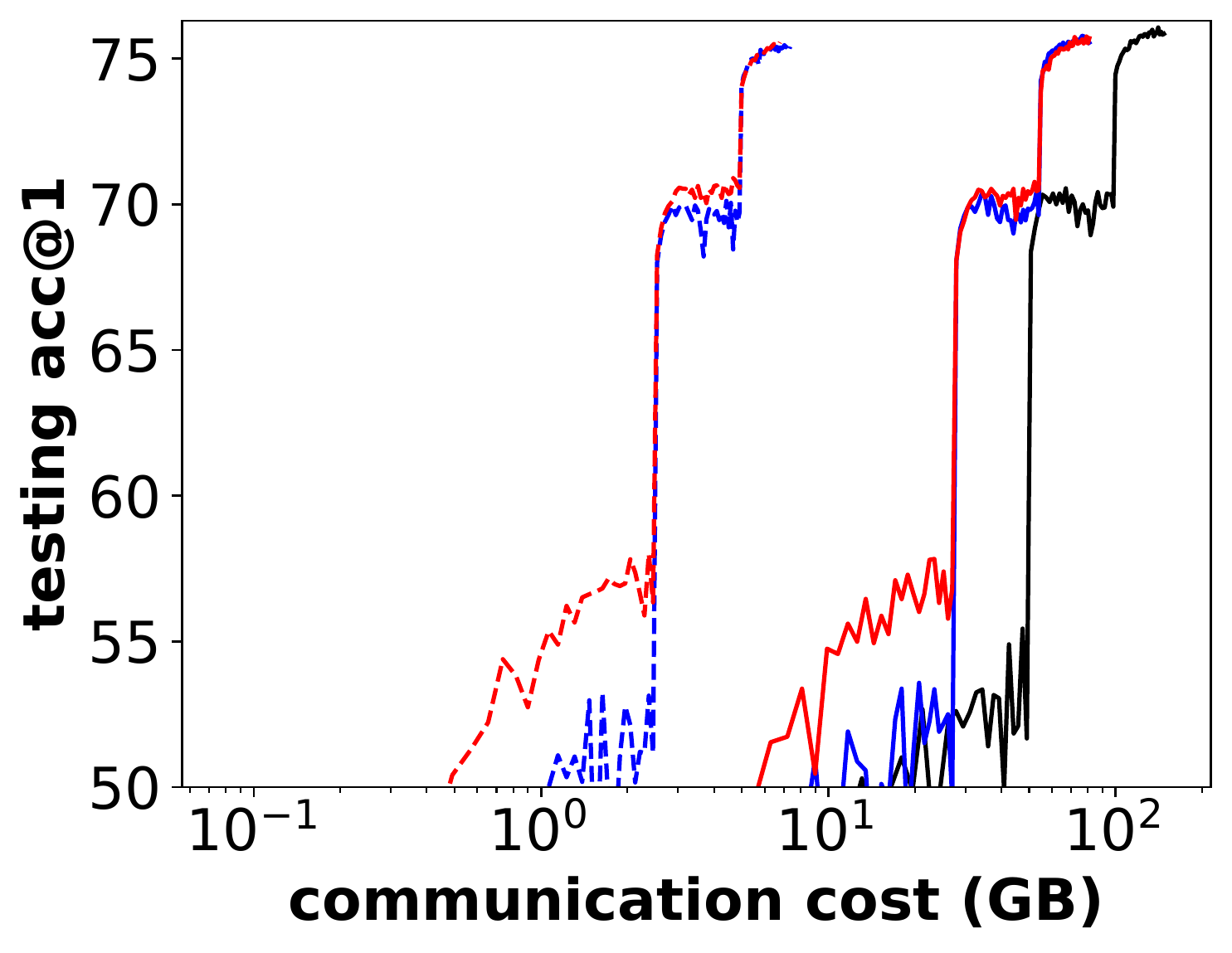}
\caption{Train ResNet-50 on ImageNet. Error averaging is compared. The top row employs 4 workers with SignSGD compression applied in the left two figures and Top-K compression applied in the right two figures. The bottom row employs 8 workers and Top-K compression, where training curves regarding epochs are shown in the left two figures and training curves regarding communication costs are shown in the right two figures.}
\label{imagenet training curve.}
\vspace{-3pt}
\end{figure*}

\begin{figure*}[t]
\centering
    \includegraphics[width=0.32\textwidth]{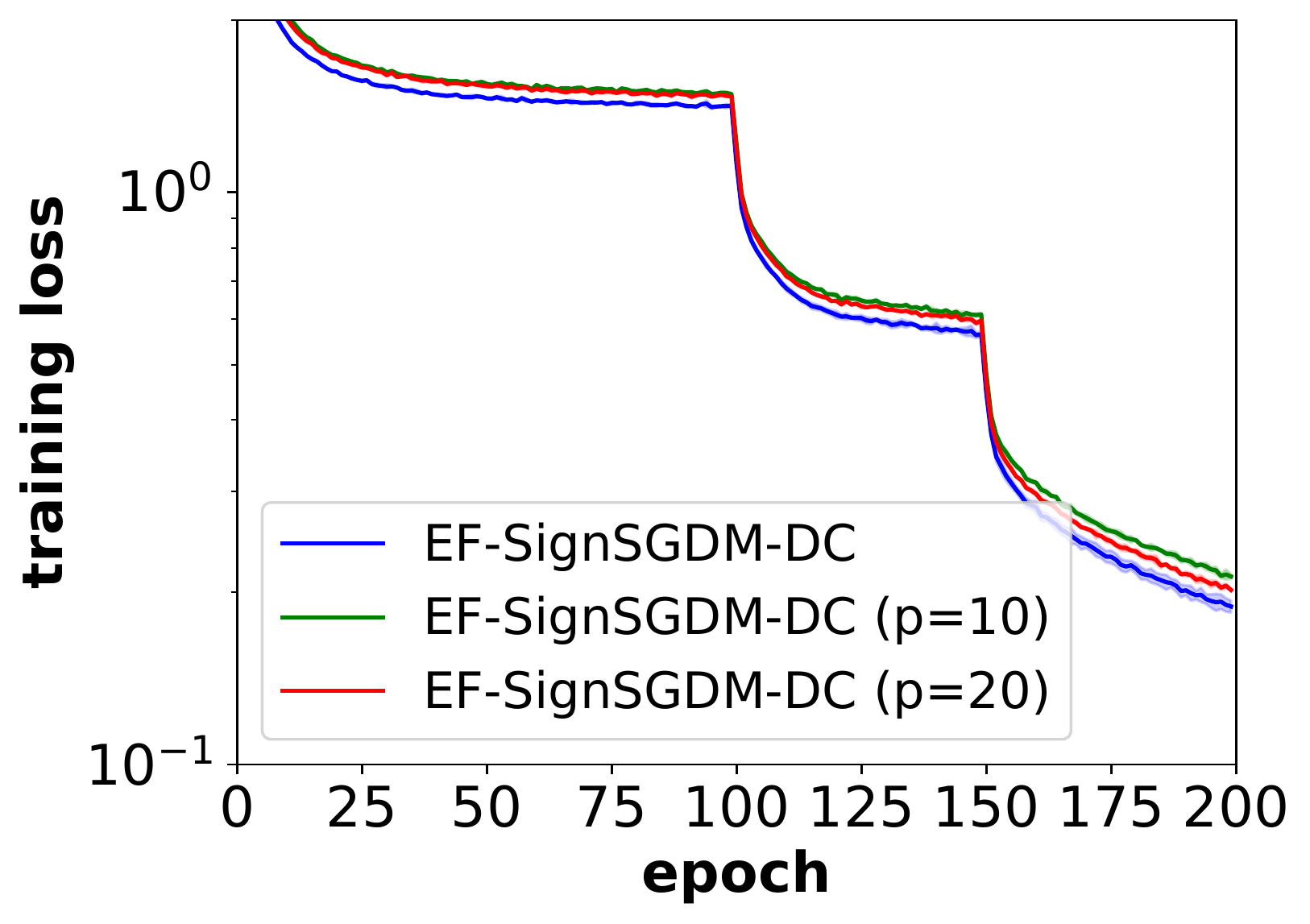}\quad
    \includegraphics[width=0.32\textwidth]{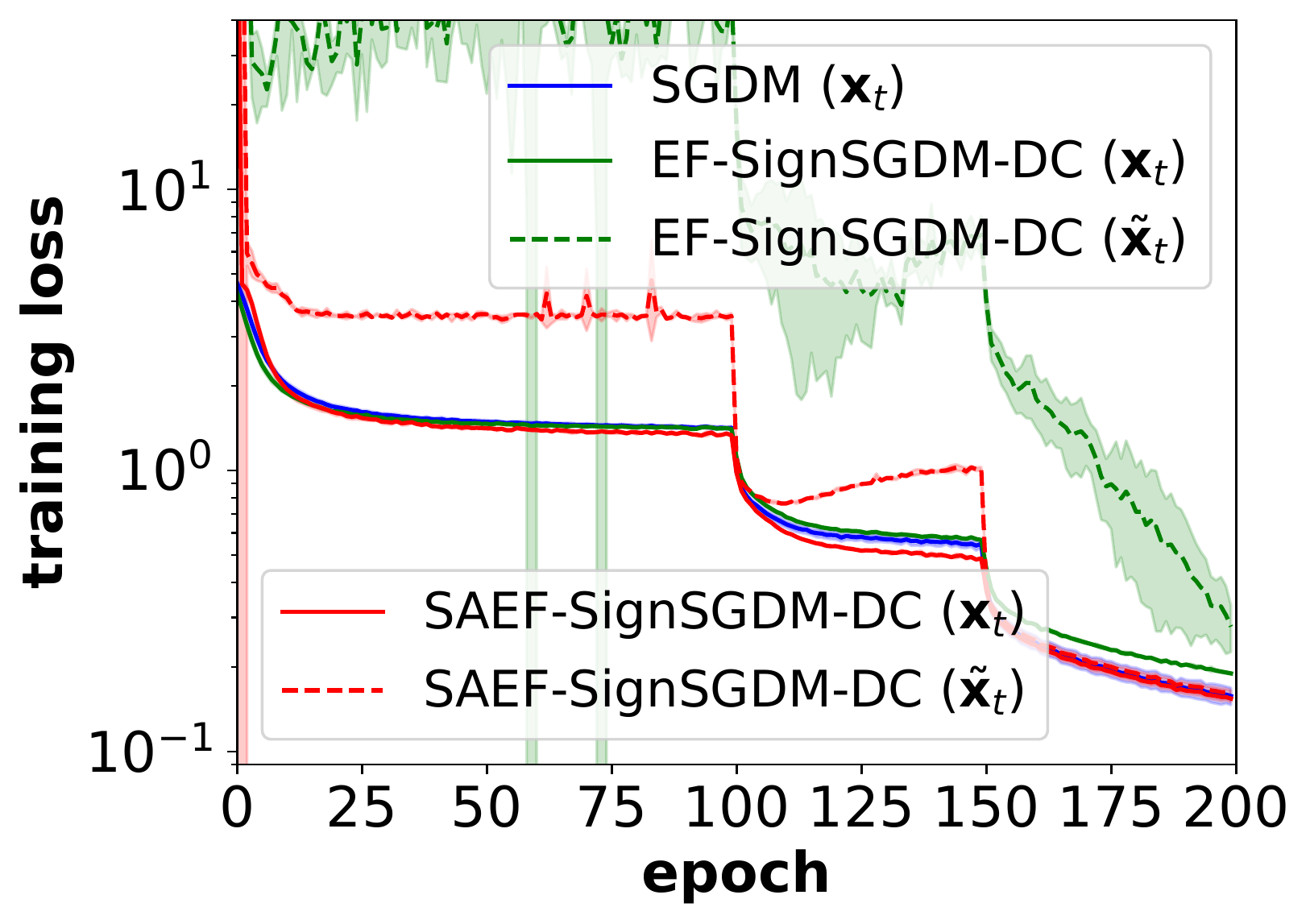}\quad
    \includegraphics[width=0.32\textwidth]{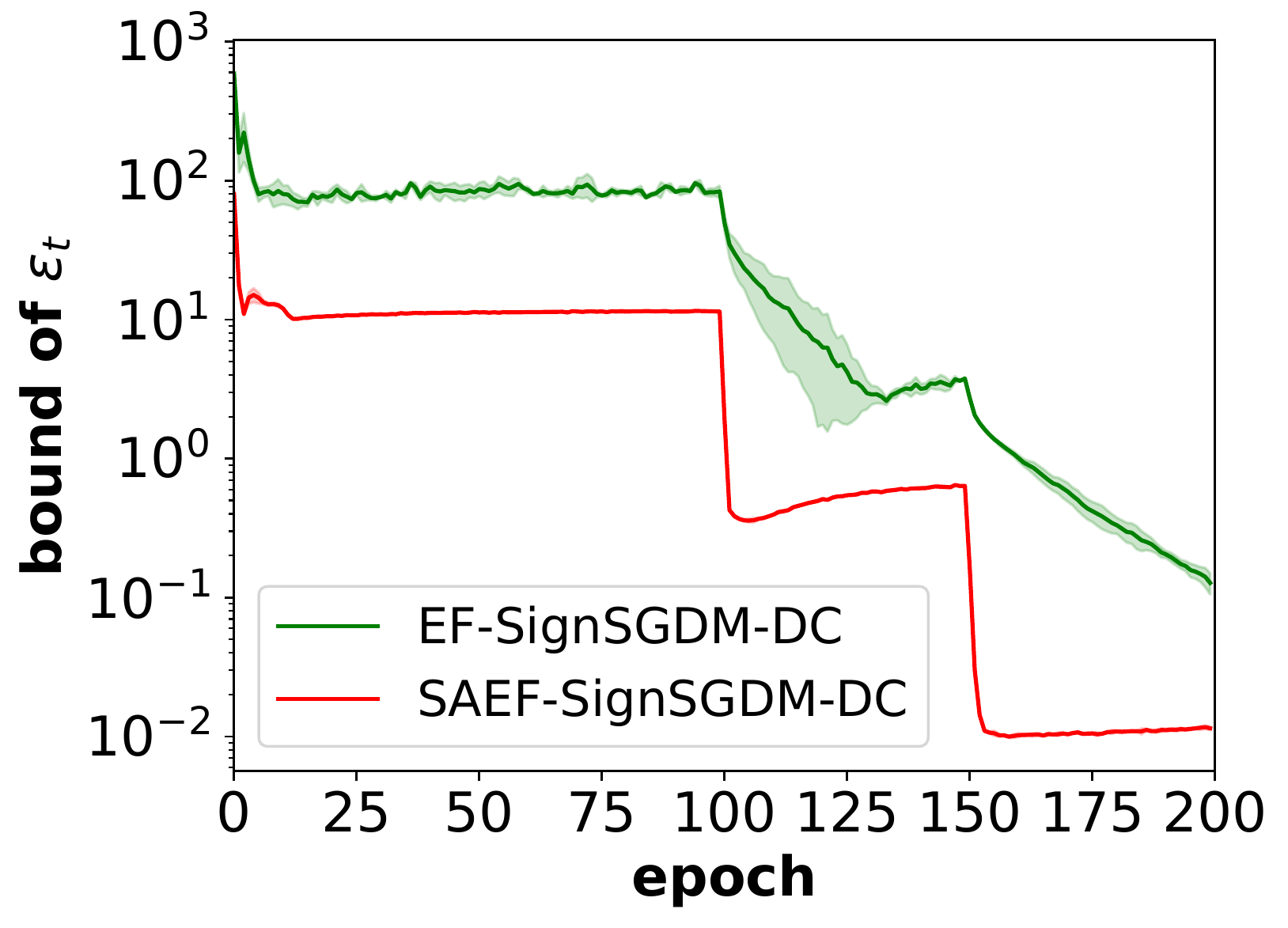}
\caption{Train ResNet-56 on CIFAR-100 with 4 workers, SignSGD compression and momentum SGD related methods. Left: error averaging in local error feedback. Middle: $\textbf{x}_t$ and $\tilde{\textbf{x}}_t$ in SAEF and local error feedback. Right: the bound of gradient mismatch ($L$ and $\sigma^2$ are ignored).}
\label{mismatch curve.}
\vspace{-8pt}
\end{figure*}

\begin{table*}[t]
\small
    \centering
    \begin{tabular}{c|c|cccccc}
    \toprule
    CIFAR-100 & EF & $p=\infty$ & $p=40$ & $p=20$ & $p=10$ & $p=5$ & $p=1$\\
    \midrule
    Top-1\% & 50.00 $\pm$ 0.70 & 60.59 $\pm$ 0.28 & 62.83 $\pm$ 0.27 & 63.96 $\pm$ 0.35 & 64.55 $\pm$ 0.20 & 65.62 $\pm$ 0.24  & 65.89 $\pm$ 0.32\\
    Top-5\% & 50.06 $\pm$ 1.17 & 60.21 $\pm$ 0.64 & 60.00 $\pm$ 0.64 & 60.94 $\pm$ 0.34 & 61.72 $\pm$ 0.13 & 61.81 $\pm$ 0.08  & 62.10 $\pm$ 0.19\\
    Top-10\% & 51.11 $\pm$ 0.24 & 57.78 $\pm$ 0.21 & 57.85 $\pm$ 0.22 & 58.22 $\pm$ 0.29 & 58.34 $\pm$ 0.38 & 58.19 $\pm$ 0.16 & 58.76 $\pm$ 0.31\\
    \bottomrule
  \end{tabular}
  \caption{Best Top-1 Testing Accuracy (\%) at epoch 100 of training ResNet-56 on CIFAR-100 using SAEF to demonstrate its faster convergence. We use Top-K sparsification (the first column shows the sparsity) and double-way compression. The second column is the result of EF (local error feedback) for comparison. We report the result in the form of (mean $\pm$ standard deviation) over 5 runs.}
\label{test acc}
\end{table*}

All experiments are implemented with PyTorch~\cite{paszke2019pytorch}. We first explain the notations of different methods ``(EF, SAEF)-(SGD, SGDM, SignSGD, SignSGDM)-(SC, DC)-(TopK)" as used in Figures \ref{cifar100 training curves.}, \ref{imagenet training curve.}, and \ref{mismatch curve.}:
\begin{itemize}
    \item Local error feedback (EF) or our proposed step ahead error feedback (SAEF).
    \item SGD, momentum SGD (SGDM), SGD with SignSGD compression (SignSGD), momentum SGD with SignSGD compression (SignSGDM).
    \item Single-way compression (SC), that is, no compression of what the server sends back to the workers, or double-way compression (DC).
    \item Whether to use Top-K gradient sparsification. If Top-K is employed, we specify the sparsity in percentage. $p=\infty$ (no error averaging) by default unless specified otherwise. All compression are performed in a layer-wise way.
\end{itemize}

\subsection{CIFAR Settings}
We train the ResNet-56~\cite{he2016deep} model with multiple workers (GPUs) on CIFAR-100~\cite{krizhevsky2009learning} image classification task. We report the mean and standard deviation metrics over 5 runs. The base learning rate is 0.1 and the total batch size is 128. The momentum constant is 0.9 and the weight decay is $5\times 10^{-4}$. For momentum SGD the model is trained for 200 epochs with a learning rate decay of 0.1 at epoch 100 and 150. For SGD the model is trained for 150 epochs with a learning rate decay of 0.1 at epoch 100 because there is barely any further improvement of the testing performance if we do a second learning rate decay. Random cropping, random flipping, and standardization are applied as data augmentation techniques.

\subsection{ImageNet Settings}
We train the ResNet-50 model with multiple workers (GPUs) on ImageNet~\cite{ILSVRC15} image classification tasks. The model is trained for 90 epochs with a learning rate decay of 0.1 at epoch 30 and 60. The base learning rate is 0.1 and the total batch size is 256. The momentum constant is 0.9 and the weight decay is $1\times 10^{-4}$. Similar data augmentation techniques as in CIFAR-100 experiments are applied.

\subsection{Performance Comparison}
\textbf{Faster Convergence.} The training curves in Figures \ref{cifar100 training curves.} and \ref{imagenet training curve.} show that employing our proposed SAEF in SGD/momentum SGD, with SignSGD/Top-K compression and single-way/double-way compression all lead to significantly faster convergence of the training loss. It is \textbf{not only faster than local error feedback but also vanilla SGD/momentum SGD with full precision gradient}. For local error feedback, we observe that its training loss is very similar to that of SGD/momentum SGD. But sometimes it may perform worse in CIFAR-100 experiments as shown in the bottom left of Figure \ref{cifar100 training curves.}, and the final training loss as shown in the top left of Figure \ref{cifar100 training curves.}. Its initial training performance can also perform worse in ImageNet experiments as shown in the bottom left of Figure \ref{imagenet training curve.}.

\textbf{Better Initial Generalization.} Although we observe a very similar final testing performance for SAEF, local error feedback, and vanilla methods, SAEF always enjoys a better testing performance before the second learning rate decay in momentum SGD experiments. The improvement is very significant, especially during the initial training. This can be crucial in the communication-constraint scenario where we need gradient compression to reduce the cost. As shown in the bottom right two figures of Figure \ref{cifar100 training curves.} and Figure \ref{imagenet training curve.}, SAEF achieves \textbf{much better training and testing performance under the same communication budget than both local error feedback and vanilla methods}. Note that in Figure \ref{imagenet training curve.}, we employ a more aggressive compression scheme for SAEF with error averaging to maintain the same communication budget. Error averaging improves SAEF (the top right two figures of Figure \ref{imagenet training curve.}) but slightly \textit{degrades local error feedback's convergence} (the left of Figure \ref{mismatch curve.}).

\textbf{Effect of Gradient Mismatch.} In the right of Figure \ref{mismatch curve.}, local error feedback features a much larger bound of gradient mismatch $\epsilon_t$ during the whole training, contributing to a worse $\tilde{\textbf{x}}_t$ and a larger gap between the training of $\textbf{x}_t$ and the auxiliary variable $\tilde{\textbf{x}}_t$ as shown in the middle of Figure \ref{mismatch curve.}. The gap is even more obvious during the initial training. By reducing this gap with SAEF we achieve faster training using compressed gradients. Note that the training curves of $\tilde{\textbf{x}}_t$ may seem poor in the initial training because we have tuned the best hyperparameters for the real trained model $\textbf{x}_t$.

\textbf{Effect of Compression Ratio and Averaging Period $p$.} Firstly, we stress that the averaging period $p$ should be \textit{large}, so that the local error $\textbf{e}^{(k)}_t$ will be communicated much more \textit{infrequently} than the gradient. As a matter of fact, we use $p=\infty$ in all our experiments except the 8-worker distributed training of ResNet-50 on ImageNet, where $p=20$ with Top-5\% gradient sparsification. To explore the effect of the different combinations of the compression ratio and averaging period, we report the top-1 testing accuracy using Top-K compression with different sparsity and vary the averaging period. We summarize it in Table \ref{test acc}. The results confirm our previous analysis that the improvement of SAEF over local error feedback gets enlarged as the compression scheme becomes more aggressive. Decreasing the averaging period usually can further accelerate the training at the cost of a larger communication budget, and it is also more obvious for an aggressive compression scheme. However, even if we do not perform error averaging ($p=\infty$), there is still a considerable improvement by using SAEF.

\section{Related Works}\label{related works}
Most existing works employ local error feedback as a standard technique in dealing with the performance loss resulting from aggressive gradient compression. We believe that they may replace local error feedback with our proposed SAEF both theoretically and empirically.

The leverage of local error feedback can be as early as \cite{seide20141} for accelerating the training of speech models. \citet{lin2018deep} proposed to locally accumulate those small gradient components until they reach a certain threshold before sending.  ECQ-SGD \cite{wu2018error} analyzed local error feedback for quantized gradients on quadratic functions. Deterioration of training performance can be observed in ECQ-SGD experiments. The Top-K compression has been proposed in \cite{strom2015scalable,aji-heafield-2017-sparse,alistarh2018convergence,stich2018sparsified}. Combine it with local error feedback and we can make each parameter get updated sooner or later. Local error feedback was first utilized to analyze and fix the testing performance loss resulting from SignSGD compression in \cite{karimireddy2019error}. \citet{zheng2019communication} later developed it for distributed momentum SGD with double-way blockwise SignSGD compression. \cite{basu2019qsparse} combined gradient compression, local error feedback, and local SGD but only considered single-way compression. Asynchronous training is also considered in \cite{basu2019qsparse}. However, all these works \textit{did not} show that we can train faster with compressed gradients without loss of performance.

We note that certain gradient compression scheme may accelerate the initial training but lead to performance loss more or less in the end. SignSGD, for example, can be faster than SGD in the beginning but quickly deteriorates in terms of the final performance. In this work, however, we have been focused on improving local error feedback with common gradient compression schemes and without performance loss.

\section{Conclusion}\label{Discussion}
In this paper, we first identified the ``gradient mismatch" problem in the local error feedback method (to the best of our knowledge, this is the first paper to systematically discuss this problem) and showed that this issue can cause performance loss in local error feedback. After that, we proposed a new SAEF (Step-Ahead Error Feedback) algorithm to train faster with compressed gradient than local error feedback and vanilla optimization methods with full precision gradient, both in terms of the performance regarding training epochs and communication costs. We theoretically show that our SAEF algorithm achieves a better convergence bound than local error feedback and empirically validate its faster convergence speed via image classification tasks. We also explore different experimental settings to confirm the scalability of SAEF.


\section*{Acknowledgements}
This work was partially supported by NSF IIS 1845666, 1852606, 1838627, 1837956, 1956002, 2040588.

\bibliography{ref.bib}

\appendix
\onecolumn
\newtheorem{aassumption}{Assumption}
\newtheorem{aproposition}{Proposition}
\newtheorem{alemma}{Lemma}
\newtheorem{atheorem}{Theorem}
\newtheorem{acorollary}{Corollary}

\section{Update of the Auxiliary Variable $\tilde{\textbf{x}}_t$}
For EF-SGDM,
\begin{equation}
\begin{split}
    \tilde{\textbf{x}}_{t+1} &\coloneqq \textbf{x}_{t+1}-(\textbf{e}_{t+1}+\frac{1}{K}\sum^{K}_{k=1}\textbf{e}^{(k)}_{t+1})\\
    &=\textbf{x}_{t+1}-(\Delta_{t+1}-\mathcal{C}(\Delta_{t+1})+\frac{1}{K}\sum^{K}_{k=1}(\Delta^{(k)}_{t+1}-\mathcal{C}(\Delta^{(k)}_{t+1})))\\
    &=\textbf{x}_t-(\Delta_{t+1}+\frac{1}{K}\sum^{K}_{k=1}(\Delta^{(k)}_{t+1}-\mathcal{C}(\Delta^{(k)}_{t+1})))\\
    &=\textbf{x}_t-(\textbf{e}_t+\frac{1}{K}\sum^{K}_{k=1}\Delta^{(k)}_{t+1})\\
    &=\textbf{x}_t-(\textbf{e}_t+\frac{1}{K}\sum^{K}_{k=1}\textbf{e}^{(k)}_t)-\eta_t\frac{1}{K}\sum^{K}_{k=1}\textbf{m}^{(k)}_{t+1}\\
    &=\tilde{\textbf{x}}_t-\frac{\eta_t}{K}\sum^{K}_{k=1}\textbf{m}^{(k)}_{t+1}\,,\\
\end{split}
\end{equation}
where the third equation is due to $\textbf{x}_{t+1}=\textbf{x}_t-\mathcal{C}(\Delta_{t+1})$, the forth equation is due to $\Delta_{t+1}=\textbf{e}_t+\frac{1}{K}\sum^{K}_{k=1}\mathcal{C}(\Delta^{(k)}_{t+1})$ and the fifth is due to $\Delta^{(k)}_{t+1}=\textbf{e}^{(k)}_t-\eta_t \textbf{m}^{(k)}_{t+1}$. Note that $\textbf{x}_t=\textbf{x}^{(k)}_t$ for all $k=1,\cdots,K$. For EF-SGD without momentum,
\begin{equation}
    \tilde{\textbf{x}}_{t+1}=\tilde{\textbf{x}}_t-\frac{\eta_t}{K}\sum^{K}_{k=1}\nabla f(\textbf{x}^{(k)}_{t};\xi^{(k)}_t)\,.
\end{equation}
Similarly for SAEF-SGDM,
\begin{equation}
    \tilde{\textbf{x}}_{t+1}=\tilde{\textbf{x}}_t-\frac{\eta_t}{K}\sum^{K}_{k=1}\textbf{m}^{(k)}_{t+1}=\tilde{\textbf{x}}_t-\frac{\eta_t}{K}\sum^{K}_{k=1}(\mu \textbf{m}^{(k)}_t+\nabla f(\textbf{x}^{(k)}_{t+\frac{1}{2}};\xi^{(k)}_t))\,.
\end{equation}
For SAEF-SGD without momentum,
\begin{equation}
    \tilde{\textbf{x}}_{t+1}=\tilde{\textbf{x}}_t-\frac{\eta_t}{K}\sum^{K}_{k=1}\nabla f(\textbf{x}^{(k)}_{t+\frac{1}{2}};\xi^{(k)}_t)\,.
\end{equation}

\section{Assumptions}

\begin{aassumption}\label{appendix:compressor}
($\delta$-approximate compressor) The compression function $\mathcal{C}(\cdot)$ : $\mathbb{R}^d\to\mathbb{R^d}$ is a $\delta$-approximate compressor for $0<\delta\leq 1$ if for all $\textbf{v}\in\mathbb{R}^d$,
\begin{equation}
    \|\mathcal{C}(\textbf{v})-\textbf{v}\|^2_2\leq (1-\delta)\|\textbf{v}\|^2_2\,.
\end{equation}
\end{aassumption}

\begin{aassumption}\label{appendix:lipschitz gradient}
($L$-Lipschitz gradient) Assume the full loss function $F(\cdot)$ is $L$-smooth, that is, $\forall \textbf{x}, \textbf{y}\in \mathbb{R}^d$ we have
\begin{equation}
    \|\nabla F(\textbf{x})-\nabla F(\textbf{y})\|_2\leq L\|\textbf{x}-\textbf{y}\|_2\,.
\end{equation}
\end{aassumption}

\begin{aassumption}\label{appendix:bounded variance}
(Bounded variance) The stochastic gradient $\nabla f(\textbf{x}^{(k)}_t;\xi^{(k)}_t)$ has bounded variance:
\begin{equation}
    \mathbb{E}\|\nabla f(\textbf{x}^{(k)}_t;\xi^{(k)}_t)-\nabla F(\textbf{x}^{(k)}_t)\|^2_2\leq \sigma^2\,.
\end{equation}
\end{aassumption}

\begin{aassumption}\label{appendix:bounded second moment}
(Bounded second moment) The full gradient is bounded:
\begin{equation}
    \|\nabla F(\textbf{x}^{(k)}_t)\|^2_2\leq M^2\,.
\end{equation}
It implies the second moment of the stochastic gradient is bounded if Assumption \ref{appendix:bounded variance} exists at the same time:
\begin{equation}
    \mathbb{E}\|\nabla f(\textbf{x}^{(k)}_t;\xi^{(k)}_t)\|^2_2\leq \sigma^2+M^2\,.
\end{equation}
\end{aassumption}

\section{Proposition 1}
\begin{aproposition}
If Assumption \ref{appendix:lipschitz gradient} and \ref{appendix:bounded variance} exist, for the the same error $\textbf{e}^{(k)}_t$ ($k=1,\cdots,K$) and $\textbf{e}_t$, the upper bound of $\epsilon_t^{(k)}$ we can prove in SAEF-SGD is smaller than that in EF-SGD if $\text{Var}(\textbf{e}^{(k)}_t)\leq \|\mathbb{E}\textbf{e}^{(k)}_t\|^2_2$.
\end{aproposition}
\begin{proof}
For SAEF-SGD,
\begin{equation}
\begin{split}
    &\frac{1}{K}\sum^{K}_{k=1}\mathbb{E}\|\textbf{e}_t+\frac{1}{K}\sum^{K}_{k^\prime=1}\textbf{e}^{(k^\prime)}_t-\textbf{e}^{(k)}_t\|^2_2\\
    &\leq 2\mathbb{E}\|\textbf{e}_t\|^2_2+\frac{2}{K}\sum^{K}_{k=1}\mathbb{E}\|\frac{1}{K}\sum^{K}_{k^\prime=1}\textbf{e}^{(k^\prime)}_t-\textbf{e}^{(k)}_t\|^2_2\\
    &= 2\mathbb{E}\|\textbf{e}_t\|^2+\frac{2}{K}\sum^{K}_{k=1}\frac{1}{K^2}\mathbb{E}\|\sum^{K}_{k^\prime=1,k^\prime\neq k}(\textbf{e}^{(k^\prime)}_t-\mathbb{E}\textbf{e}^{(k^\prime)}_t+\mathbb{E}\textbf{e}^{(k)}_t-\textbf{e}^{(k)}_t)\|^2\\
    &= 2\mathbb{E}\|\textbf{e}_t\|^2+\frac{2}{K^3}\sum^{K}_{k=1}\left(\sum^{K}_{k^\prime=1,k^\prime\neq k}\text{Var}(\textbf{e}^{(k^\prime)}_t)+(K-1)^2\text{Var}(\textbf{e}^{(k)}_t)\right)\\
    &= 2\mathbb{E}\|\textbf{e}_t\|^2+\frac{2(K-1)}{K^2}\sum^{K}_{k=1}\text{Var}(\textbf{e}^{(k)}_t)\,.\\
\end{split}
\end{equation}
The second equation is due to the unbiased gradient $\mathbb{E}_t\nabla f(\textbf{x}^{(k)}_t;\xi^{(k)}_t)=\nabla F(\textbf{x}^{(k)}_t)=\nabla F(\textbf{x}_t)$, which leads to $\forall 1\leq i, j\leq K, \mathbb{E}_{0,1,\cdots,t-1}\textbf{e}^{(i)}_t=\mathbb{E}_{0,1,\cdots,t-1}\textbf{e}^{(j)}_t$. The third equation is due to that $\{\mathbb{E}_{0,1,...,t-2}\textbf{e}^{(k)}_t\}_{k=1,2,\cdots,K}$ are independent.
For EF-SGD,
\begin{equation}
\begin{split}
    \mathbb{E}\|\textbf{e}_t+\frac{1}{K}\sum^{K}_{k=1}\textbf{e}^{(k)}_t\|^2\leq 2\mathbb{E}\|\textbf{e}_t\|^2+2\mathbb{E}\|\frac{1}{K}\sum^{K}_{k=1}\textbf{e}^{(k)}_t\|^2 = 2\mathbb{E}\|\textbf{e}_t\|^2+\frac{2}{K^2}\sum^{K}_{k=1}\text{Var}(\textbf{e}^{(k)}_t) + \frac{2}{K}\sum^{K}_{k=1}\|\mathbb{E}\textbf{e}^{(k)}_t\|^2\,.
\end{split}
\end{equation}
Because $\text{Var}(\textbf{e}^{(k)}_t)\leq \|\mathbb{E}\textbf{e}^{(k)}_t\|^2$,
\begin{equation}
    \frac{2(K-2)}{K^2}\sum^{K}_{k=1}\text{Var}(\textbf{e}^{(k)}_t)\leq \frac{2(K-2)}{K^2}\sum^{K}_{k=1}\|\mathbb{E}\textbf{e}^{(k)}_t\|^2 < \frac{2}{K}\sum^{K}_{k=1}\mathbb{E}\|\textbf{e}^{(k)}_t\|^2\,.
\end{equation}
Because $\mathbb{E}\|\textbf{e}^{(k)}_t\|^2$ can be bounded in convergence analysis, we can always prove a better upper bound of $\epsilon_t$ for SAEF-SGD than EF-SGD. For single-way compression where $\textbf{e}_t=\textbf{0}$, this proposition still exists.
\end{proof}

\section{Lemmas}

\begin{alemma}\label{(em)^2}
Under Assumption \ref{appendix:bounded variance} and \ref{appendix:bounded second moment}, we have
\begin{equation}
    \mathbb{E}\|\textbf{m}^{(k)}_{t+1}\|^2\leq \frac{M^2+\sigma^2}{(1-\mu)^2}\,.
\end{equation}
\end{alemma}
\begin{proof}
\begin{equation}
\begin{split}
    \mathbb{E}\|\textbf{m}^{(k)}_{t+1}\|^2 &=\mathbb{E}\|\mu \textbf{m}^{(k)}_t+\nabla f(\textbf{x}^{(k)}_{t+\frac{1}{2}};\xi^{(k)}_t)\|^2\\
    &= \mathbb{E}\|\sum^{t}_{i=0}\mu^{t-i}\nabla f(\textbf{x}^{(k)}_{i+\frac{1}{2}};\xi^{(k)}_i)\|^2\\
    &=(\sum^{t}_{i=0}\mu^{t-i})^2\mathbb{E}\left\|\frac{\sum^{t}_{i=0}\mu^{t-i}\nabla f(\textbf{x}^{(k)}_{i+\frac{1}{2}};\xi^{(k)}_i)}{\sum^{t}_{i=0}\mu^{t-i}}\right\|^2\\
    &\leq (\sum^{t}_{i=0}\mu^{t-i})\sum^{t}_{i=0}\mu^{t-i}\mathbb{E}\|\nabla f(\textbf{x}^{(k)}_{i+\frac{1}{2}};\xi^{(k)}_i)\|^2 \leq \frac{M^2+\sigma^2}{(1-\mu)^2}\,.\\
\end{split}
\end{equation}
\end{proof}

\begin{alemma}\label{(ee)^2}
Under Assumption \ref{appendix:bounded variance}, \ref{appendix:bounded second moment} and \ref{appendix:compressor}, $\forall \beta_1>0$ and $(1-\delta)(1+\beta_1)<1$, i.e. $0<\beta_1<\frac{\delta}{1-\delta}$, we have
\begin{equation}
    \mathbb{E}\|\textbf{e}^{(k)}_{t+1}\|^2\leq\frac{(1-\delta)(1+\frac{1}{\beta_1})}{1-(1-\delta)(1+\beta_1)}\frac{\eta_{max}^2(M^2+\sigma^2)}{(1-\mu)^2}\,.
\end{equation}
The bound is minimum when $\beta_1=-1+\frac{1}{\sqrt{1-\delta}}$, which leads to
\begin{equation}
    \mathbb{E}\|\textbf{e}^{(k)}_{t+1}\|^2 \leq \frac{1-\delta}{(1-\sqrt{1-\delta})^2} \frac{\eta_{max}^2(M^2+\sigma^2)}{(1-\mu)^2}\,.
\end{equation}
\end{alemma}
\begin{proof}
\begin{equation}
\begin{split}
    \mathbb{E}\|\textbf{e}^{(k)}_{t+1}\|^2 &= \mathbb{E}\|\Delta^{(k)}_{t+1}-\mathcal{C}(\Delta^{(k)}_{t+1})\|^2\\
    &\leq (1-\delta)\mathbb{E}\|\Delta^{(k)}_{t+1}\|^2\\
    &=(1-\delta)\mathbb{E}\|\textbf{e}^{(k)}_t+\eta_t\textbf{m}^{(k)}_{t+1}\|^2\\
    &\leq (1-\delta)(1+\beta_1)\mathbb{E}\|\textbf{e}^{(k)}_t\|^2+(1-\delta)(1+\frac{1}{\beta_1})\eta_t^2\mathbb{E}\|\textbf{m}^{(k)}_{t+1}\|^2\\
    &\leq (1-\delta)(1+\beta_1)\mathbb{E}\|\textbf{e}^{(k)}_t\|^2 + (1-\delta)(1+\frac{1}{\beta_1})\eta_t^2\frac{M^2+\sigma^2}{(1-\mu)^2}\\
\end{split}
\end{equation}
The last inequality follows Lemma \ref{(em)^2}. Then,
\begin{equation}
\begin{split}
    \mathbb{E}\|\textbf{e}^{(k)}_{t+1}\|^2 &\leq \sum^{t}_{i=0}[(1-\delta)(1+\beta_1)]^{t-i}(1-\delta)(1+\frac{1}{\beta_1})\frac{\eta_t^2(M^2+\sigma^2)}{(1-\mu)^2}\\
    &\leq \frac{(1-\delta)(1+\frac{1}{\beta_1})}{1-(1-\delta)(1+\beta_1)}\frac{\eta_{max}^2(M^2+\sigma^2)}{(1-\mu)^2}\,.\\
\end{split}
\end{equation}
Let $h(\beta_1)\coloneqq \frac{(1-\delta)(1+\frac{1}{\beta_1})}{1-(1-\delta)(1+\beta_1)}$,
\begin{equation}
\begin{split}
    \frac{d}{d\beta_1}h(\beta_1) &= \frac{-(1-\delta)\frac{1}{\beta_1^2}(1-(1-\delta)(1+\beta_1))-(1-\delta)(1+\frac{1}{\beta_1})(-(1-\delta))}{(1-(1-\delta)(1+\beta_1))^2}\\
    &=\frac{1-\delta}{\beta_1^2(1-(1-\delta)(1+\beta_1))^2}((1-\delta)\beta_1^2+2(1-\delta)\beta_1-\delta)\,.\\
\end{split}
\end{equation}
As $0<-1+\frac{1}{\sqrt{1-\delta}}<\frac{\delta}{1-\delta}$, we have
\begin{equation}
    h(\beta_1)\geq h(-1+\frac{1}{\sqrt{1-\delta}})=\frac{1-\delta}{(1-\sqrt{1-\delta})^2}\,.
\end{equation}
Thus,
\begin{equation}
    \mathbb{E}\|\textbf{e}^{(k)}_{t+1}\|^2 \leq \frac{1-\delta}{(1-\sqrt{1-\delta})^2} \frac{\eta_{max}^2(M^2+\sigma^2)}{(1-\mu^2)}\,.
\end{equation}
\end{proof}

\begin{alemma}\label{worker ee^2}
Under the same conditions of Lemma \ref{(ee)^2}, we have
\begin{equation}
    \frac{1}{K}\sum^{K}_{k=1}\mathbb{E}\|\frac{1}{K}\sum^{K}_{k=1}\textbf{e}^{(k)}_t-\textbf{e}^{(k)}_t\|^2 \leq \frac{K-1}{K}\frac{(1-\delta)(1+\frac{1}{\beta_1})}{1-(1-\delta)(1+\beta_1)}\frac{\eta_{max}^2(M^2+\sigma^2)}{(1-\mu)^2}\,.
\end{equation}
\end{alemma}
\begin{proof}
\begin{equation}
\begin{split}
    \frac{1}{K}\sum^{K}_{k=1}\mathbb{E}\|\frac{1}{K}\sum^{K}_{k=1}\textbf{e}^{(k)}_t-\textbf{e}^{(k)}_t\|^2 &= \frac{1}{K^3}\sum^{K}_{k=1}\mathbb{E}\|\sum^{K}_{j=1}(\textbf{e}^j_t-\mathbb{E}\textbf{e}^j_t+\mathbb{E}\textbf{e}^{(k)}_t-\textbf{e}^{(k)}_t)\|^2\\
    &\leq \frac{1}{K^3}\sum^{K}_{k=1}\left(\sum^{K}_{j=1,j\neq k}\text{Var}(\textbf{e}^j_t)+(K-1)^2\text{Var}(\textbf{e}^{(k)}_t)\right)\\
    &= \frac{K-1}{K^2}\sum^{K}_{k=1}\text{Var}(\textbf{e}^{(k)}_t)\\
    &< \frac{K-1}{K^2}\sum^{K}_{k=1}\mathbb{E}\|\textbf{e}^{(k)}_t\|^2\\
    &\leq \frac{K-1}{K}\frac{(1-\delta)(1+\frac{1}{\beta_1})}{1-(1-\delta)(1+\beta_1)}\frac{\eta_{max}^2(M^2+\sigma^2)}{(1-\mu)^2}\\
\end{split}
\end{equation}
The first inequality is due to $\mathbb{E}_{t-1}\langle \textbf{e}^i_t,\textbf{e}^{(k)}_t\rangle=\mathbb{E}_{t-1}\|\textbf{e}^i_t\|^2=\mathbb{E}_{t-1}\|\textbf{e}^{(k)}_t\|^2$. The second inequality follows Lemma \ref{(ee)^2}.
\end{proof}

\begin{alemma}\label{server ee^2}
Under the same conditions of Lemma \ref{(ee)^2}, we have
\begin{equation}
    \mathbb{E}\|\textbf{e}_{t+1}\|^2\leq\frac{1-\delta}{(1-\sqrt{1-\delta})^2} \frac{2(2-\delta)(1+\frac{1}{\beta_1})}{1-(1-\delta)(1+\beta_1)}\frac{\eta_{max}^2(M^2+\sigma^2)}{(1-\mu)^2}\,.
\end{equation}
\end{alemma}
\begin{proof}
\begin{equation}
\begin{split}
    \mathbb{E}\|\textbf{e}_{t+1}\|^2 &= \mathbb{E}\|\Delta_{t+1}-\mathcal{C}(\Delta_{t+1})\|^2 \leq (1-\delta)\mathbb{E}\|\Delta_{t+1}\|^2\\
    &=(1-\delta)\mathbb{E}\|\textbf{e}_t+\frac{1}{K}\sum^{K}_{k=1}\mathcal{C}(\Delta^{(k)}_{t+1})\|^2\\
    &\leq (1-\delta)(1+\beta_2)\mathbb{E}\|\textbf{e}_t\|^2+(1-\delta)(1+\frac{1}{\beta_2})\mathbb{E}\|\frac{1}{K}\sum^{K}_{k=1}\mathcal{C}(\Delta^{(k)}_{t+1})\|^2\\
\end{split}
\end{equation}
The last term
\begin{equation}
\begin{split}
    \mathbb{E}\|\frac{1}{K}\sum^{K}_{k=1}\mathcal{C}(\Delta^{(k)}_{t+1})\|^2 & \leq 2\mathbb{E}\|\frac{1}{K}\sum^{K}_{k=1}\Delta^{(k)}_{t+1}\|^2 + 2\mathbb{E}\|\frac{1}{K}\sum^{K}_{k=1}(\mathcal{C}(\Delta^{(k)}_{t+1})-\Delta^{(k)}_{t+1})\|^2\\
    &\leq \frac{2}{K}\sum^{K}_{k=1}\mathbb{E}\|\Delta^{(k)}_{t+1}\|^2+\frac{2}{K}\sum^{K}_{k=1}\mathbb{E}\|\mathcal{C}(\Delta^{(k)}_{t+1})-\Delta^{(k)}_{t+1}\|^2\\
    &\leq \frac{2(2-\delta)}{K}\sum^{K}_{k=1}\mathbb{E}\|\Delta^{(k)}_{t+1}\|^2\,.\\
\end{split}
\end{equation}
According to Lemma \ref{(ee)^2},
\begin{equation}
\begin{split}
    \mathbb{E}\|\Delta^{(k)}_{t+1}\|^2 &\leq (1+\beta_1)\mathbb{E}\|\textbf{e}^{(k)}_t\|^2 + (1+\frac{1}{\beta_1})\eta_t^2\frac{M^2+\sigma^2}{(1-\mu)^2}\\
    &\leq \frac{(1+\frac{1}{\beta_1})}{1-(1-\delta)(1+\beta_1)}\frac{\eta_{max}^2(M^2+\sigma^2)}{(1-\mu)^2}\,.\\
\end{split}
\end{equation}
Combine the above inequalities,
\begin{equation}
\begin{split}
    \mathbb{E}\|\textbf{e}_{t+1}\|^2 &\leq \sum^{t}_{i=0}[(1-\delta)(1+\beta_2)]^{t-i}(1-\delta)(1+\frac{1}{\beta_2}) \cdot 2(2-\delta)\frac{(1+\frac{1}{\beta_1})}{1-(1-\delta)(1+\beta_1)}\frac{\eta_{max}^2(M^2+\sigma^2)}{(1-\mu)^2}\\
    &\leq \frac{(1-\delta)(1+\frac{1}{\beta_2})}{1-(1-\delta)(1+\beta_2)} \frac{2(2-\delta)(1+\frac{1}{\beta_1})}{1-(1-\delta)(1+\beta_1)}\frac{\eta_{max}^2(M^2+\sigma^2)}{(1-\mu)^2}\,.
\end{split}
\end{equation}
Let $\beta_2=-1+\frac{1}{\sqrt{1-\delta}}$ (following Lemma \ref{(ee)^2}),
\begin{equation}
    \mathbb{E}\|\textbf{e}_{t+1}\|^2\leq\frac{1-\delta}{(1-\sqrt{1-\delta})^2} \frac{2(2-\delta)(1+\frac{1}{\beta_1})}{1-(1-\delta)(1+\beta_1)}\frac{\eta_{max}^2(M^2+\sigma^2)}{(1-\mu)^2}\,.
\end{equation}
\end{proof}

\begin{alemma}\label{appendix:mismatch bound}
Under Assumptions \ref{appendix:bounded variance}, \ref{appendix:bounded second moment} and \ref{appendix:compressor}, we have
\begin{equation}
    \frac{1}{K}\sum^{K}_{k=1}\mathbb{E}\|\textbf{e}_t+\frac{1}{K}\sum^{K}_{k=1}\textbf{e}^{(k)}_t-\textbf{e}^{(k)}_t\|^2 \leq \left(\frac{2(1+\delta)(2-\delta)}{(1-\sqrt{1-\delta})^2}+\frac{1+\delta}{\delta}\right)\frac{1-\delta}{(1-\sqrt{1-\delta})^2}\frac{\eta_{max}^2(M^2+\sigma^2)}{(1-\mu)^2}\,.
\end{equation}
\end{alemma}
\begin{proof}
Combine the results of Lemmas \ref{worker ee^2} and \ref{server ee^2},
\begin{equation}
\begin{split}
    &\frac{1}{K}\sum^{K}_{k=1}\mathbb{E}\|\textbf{e}_t+\frac{1}{K}\sum^{K}_{k=1}\textbf{e}^{(k)}_t-\textbf{e}^{(k)}_t\|^2\\
    &\leq (1+\delta)\mathbb{E}\|\textbf{e}
    _t\|^2 + \frac{1+\frac{1}{\delta}}{K}\sum^{K}_{k=1}\mathbb{E}\|\frac{1}{K}\sum^{K}_{k=1}\textbf{e}^{(k)}_t-\textbf{e}^{(k)}_t\|^2\\
    &\leq \left[\frac{1-\delta}{(1-\sqrt{1-\delta})^2}2(1+\delta)(2-\delta) + \frac{1+\delta}{\delta}(1-\delta)\right] \frac{1+\frac{1}{\beta_1}}{1-(1-\delta)(1+\beta_1)}\frac{\eta_{max}^2(M^2+\sigma^2)}{(1-\mu)^2}\,.
\end{split}
\end{equation}
The bound also achieves minimum when $\beta_1=-1+\frac{1}{\sqrt{1-\delta}}$. Thus,
\begin{equation}
    \frac{1}{K}\sum^{K}_{k=1}\mathbb{E}\|\textbf{e}_t+\frac{1}{K}\sum^{K}_{k=1}\textbf{e}^{(k)}_t-\textbf{e}^{(k)}_t\|^2 \leq \left(\frac{2(1+\delta)(2-\delta)}{(1-\sqrt{1-\delta})^2}+\frac{1+\delta}{\delta}\right)\frac{1-\delta}{(1-\sqrt{1-\delta})^2}\frac{\eta_{max}^2(M^2+\sigma^2)}{(1-\mu)^2}\,.
\end{equation}
\end{proof}

\section{Convergence of SAEF-SGD}

\begin{atheorem}\label{appendix:saef-sgd auxiliary convergence}
If Assumptions \ref{appendix:lipschitz gradient}, \ref{appendix:bounded variance}, \ref{appendix:bounded second moment} and \ref{appendix:compressor} exist, and the learning rate $0<\eta_t=\eta<\frac{3}{4L}$ for all $t=0,\cdots,T-1$, for SAEF-SGD we have
\begin{equation}
\begin{split}
    \min_{t=0,\cdots,T-1}\mathbb{E}\|\nabla F(\tilde{\textbf{x}}_t)\|^2 &\leq \frac{4[F(\tilde{\textbf{x}}_0)-F(\tilde{\textbf{x}}^*)]}{\eta(3-4\eta L)T} + \frac{2\eta L\sigma^2}{(3-4\eta L)K}\\
    &\quad + \left(\frac{2(1+\delta)(2-\delta)}{(1-\sqrt{1-\delta})^2}+\frac{1+\delta}{\delta}\right)  \frac{1-\delta}{(1-\sqrt{1-\delta})^2} \frac{4(\eta L+1)\eta^2L^2}{3-4\eta L}(M^2+\sigma^2)\,.\\
\end{split}
\end{equation}
\end{atheorem}
\begin{proof}
\begin{equation}
\begin{split}
    F(\tilde{\textbf{x}}_{t+1}) &\leq F(\tilde{\textbf{x}}_t)+\langle \nabla F(\tilde{\textbf{x}}_t), \tilde{\textbf{x}}_{t+1}-\tilde{\textbf{x}}_t \rangle + \frac{L}{2}\|\tilde{\textbf{x}}_{t+1}-\tilde{\textbf{x}}_t\|^2\\
    &=F(\tilde{\textbf{x}}_t)-\langle \nabla F(\tilde{\textbf{x}}_t), \frac{\eta}{K}\sum^{K}_{k=1}\nabla f(\textbf{x}^{(k)}_{t+\frac{1}{2}};\xi^{(k)}_t) \rangle + \frac{\eta^2L}{2K^2}\|\sum^{K}_{k=1}\nabla f(\textbf{x}^{(k)}_{t+\frac{1}{2}};\xi^{(k)}_t)\|^2\\
\end{split}
\end{equation}
Take expectation at iteration $t$,
\begin{equation}
\begin{split}
    \mathbb{E}F(\tilde{\textbf{x}}_{t+1}) &= F(\tilde{\textbf{x}}_t) \underbrace{- \langle \nabla F(\tilde{\textbf{x}}_t),\frac{\eta}{K}\sum^{K}_{k=1}\nabla F(\textbf{x}^{(k)}_{t+\frac{1}{2}})\rangle}_\text{\textcircled{1}} + \underbrace{\frac{\eta^2L}{2K^2}\mathbb{E}\|\sum^{K}_{k=1}\nabla f(\textbf{x}^{(k)}_{t+\frac{1}{2}};\xi^{(k)}_t)\|^2}_\text{\textcircled{2}}\,.\\
\end{split}
\end{equation}
Firstly we consider \textcircled{1}.
\begin{equation}
\begin{split}
    -\langle \nabla F(\tilde{\textbf{x}}_t),\frac{\eta}{K}\sum^{K}_{k=1}\nabla F(\textbf{x}^{(k)}_{t+\frac{1}{2}})\rangle &= - \langle \nabla F(\tilde{\textbf{x}}_t), \eta\nabla F(\tilde{\textbf{x}}_t) + \frac{\eta}{K}\sum^{K}_{k=1}\nabla F(\textbf{x}^{(k)}_{t+\frac{1}{2}})-\eta\nabla F(\tilde{\textbf{x}}_t)\rangle\\
    &=-\eta\|\nabla F(\tilde{\textbf{x}}_t)\|^2-\eta\langle \nabla F(\tilde{\textbf{x}}_t),\frac{1}{K}\sum^{K}_{k=1}\nabla F(\textbf{x}^{(k)}_{t+\frac{1}{2}})-\nabla F(\tilde{\textbf{x}}_t)\rangle\\
    &\leq -\eta\|\nabla F(\tilde{\textbf{x}}_t)\|^2+\eta\frac{\rho}{2}\|\nabla F(\tilde{\textbf{x}}_t)\|^2+\eta\frac{1}{2\rho}\|\frac{1}{K}\sum^{K}_{k=1}\nabla F(\textbf{x}^{(k)}_{t+\frac{1}{2}})-\nabla F(\tilde{\textbf{x}}_t)\|^2\\
    &\leq -\eta(1-\frac{\rho}{2})\|\nabla F(\tilde{\textbf{x}}_t)\|^2 + \frac{\eta}{2\rho K}\sum^{K}_{k=1}\|\nabla F(\textbf{x}^{(k)}_{t+\frac{1}{2}})-\nabla F(\tilde{\textbf{x}}_t)\|^2\\
    &\leq -\eta(1-\frac{\rho}{2})\|\nabla F(\tilde{\textbf{x}}_t)\|^2 + \frac{\eta L^2}{2\rho K}\sum^{K}_{k=1}\|\textbf{e}_t+\frac{1}{K}\sum^{K}_{k=1}\textbf{e}^{(k)}_t-\textbf{e}^{(k)}_t\|^2
\end{split}
\end{equation}
Then we consider \textcircled{2}.
\begin{equation}
\begin{split}
    \frac{\eta^2L}{2K^2}\mathbb{E}\|\sum^{K}_{k=1}\nabla f(\textbf{x}^{(k)}_{t+\frac{1}{2}};\xi^{(k)}_t)\|^2 &\leq \frac{\eta^2L}{2K^2}\|\sum^{K}_{k=1}\nabla F(\textbf{x}^{(k)}_{t+\frac{1}{2}})\|^2 + \frac{\eta^2L\sigma^2}{2K}\\
    &\leq \frac{\eta^2L}{2K}\sum^{K}_{k=1}\|\nabla F(\textbf{x}^{(k)}_{t+\frac{1}{2}})\|^2+\frac{\eta^2L\sigma^2}{2K}\\
    &=\frac{\eta^2L}{2K}\sum^{K}_{k=1}\|\nabla F(\textbf{x}^{(k)}_{t+\frac{1}{2}})-\nabla F(\tilde{\textbf{x}}_t)+\nabla F(\tilde{\textbf{x}}_t)\|^2+\frac{\eta^2L\sigma^2}{2K}\\
    &\leq \frac{\eta^2L}{K}\sum^{K}_{k=1}\|\nabla F(\textbf{x}^{(k)}_{t+\frac{1}{2}})-\nabla F(\tilde{\textbf{x}}_t)\|^2+\frac{\eta^2L}{K}\sum^{K}_{k=1}\|\nabla F(\tilde{\textbf{x}}_t)\|^2+\frac{\eta^2L\sigma^2}{2K}\\
    &\leq \frac{\eta^2L^3}{K}\sum^{K}_{k=1}\|\textbf{e}_t+\frac{1}{K}\sum^{K}_{k=1}\textbf{e}^{(k)}_t-\textbf{e}^{(k)}_t\|^2 + \eta^2L\|\nabla F(\tilde{\textbf{x}}_t)\|^2 + \frac{\eta^2L\sigma^2}{2K}\\
\end{split}
\end{equation}
Combine them and we have
\begin{equation}
\begin{split}
    \mathbb{E}F(\tilde{\textbf{x}}_{t+1})-F(\tilde{\textbf{x}}_t) & \leq -\eta(1-\frac{\rho}{2}-\eta L)\|\nabla F(\tilde{\textbf{x}}_t)\|^2 + \frac{\eta L^2}{K}(\eta L+\frac{1}{2\rho})\sum^{K}_{k=1}\|\textbf{e}_t+\frac{1}{K}\sum^{K}_{k=1}\textbf{e}^{(k)}_t-\textbf{e}^{(k)}_t\|^2 + \frac{\eta^2L\sigma^2}{2K}\,.\\
\end{split}
\end{equation}
Sum from $t=0$ to $T-1$ and take the total expectation,
\begin{equation}
\begin{split}
    &F(\tilde{\textbf{x}}^*)-F(\tilde{\textbf{x}}_0) \leq \mathbb{E}(\sum^{T-1}_{t=0}(\mathbb{E}F(\tilde{\textbf{x}}_{t+1})-F(\tilde{\textbf{x}}_t)))\\
    &\leq -\eta(1-\frac{\rho}{2}-\eta L)\sum^{T-1}_{t=0}\mathbb{E}\|\nabla F(\tilde{\textbf{x}}_t)\|^2 + (\eta L+\frac{1}{2\rho})\eta L^2\sum^{T-1}_{t=0}\frac{1}{K}\sum^{K}_{k=1}\mathbb{E}\|\textbf{e}_t+\frac{1}{K}\sum^{K}_{k=1}\textbf{e}^{(k)}_t-\textbf{e}^{(k)}_t\|^2+\frac{\eta^2L\sigma^2T}{2K}\,.\\
\end{split}
\end{equation}
After rearranging,
\begin{equation}
\begin{split}
    &\min_{t=0,\cdots,T-1}\mathbb{E}\|\nabla F(\tilde{\textbf{x}}_t)\|^2 \leq \frac{1}{T}\sum^{T-1}_{t=0}\mathbb{E}\|\nabla F(\tilde{\textbf{x}}_t)\|^2\\
    &\leq \frac{F(\tilde{\textbf{x}}_0)-F(\tilde{\textbf{x}}^*)}{\eta(1-\frac{\rho}{2}-\eta L)T} + \frac{(\eta L+\frac{1}{2\rho})L^2}{(1-\frac{\rho}{2}-\eta L)T}\sum^{T-1}_{t=0}\frac{1}{K}\sum^{K}_{k=1}\mathbb{E}\|\textbf{e}_t+\frac{1}{K}\sum^{K}_{k=1}\textbf{e}^{(k)}_t-\textbf{e}^{(k)}_t\|^2 +\frac{\eta L\sigma^2}{2(1-\frac{\rho}{2}-\eta L)K}\\
    &\leq \frac{F(\tilde{\textbf{x}}_0)-F(\tilde{\textbf{x}}^*)}{\eta(1-\frac{\rho}{2}-\eta L)T} + \frac{\eta L\sigma^2}{2(1-\frac{\rho}{2}-\eta L)K}\\
    &\quad +\frac{(\eta L+\frac{1}{2\rho})L^2}{1-\frac{\rho}{2}-\eta L} \left(\frac{2(1+\delta)(2-\delta)}{(1-\sqrt{1-\delta})^2}+\frac{1+\delta}{\delta}\right)\frac{1-\delta}{(1-\sqrt{1-\delta})^2}\frac{\eta^2(M^2+\sigma^2)}{(1-\mu)^2}\,,\\
\end{split}
\end{equation}
where the third inequality follows Lemma \ref{appendix:mismatch bound}. Let $\rho=\frac{1}{2}$ (or some other appropriate value),
\begin{equation}
\begin{split}
    \min_{t=0,\cdots,T-1}\mathbb{E}\|\nabla F(\tilde{\textbf{x}}_t)\|^2 &\leq \frac{4[F(\tilde{\textbf{x}}_0)-F(\tilde{\textbf{x}}^*)]}{\eta(3-4\eta L)T} + \frac{2\eta L\sigma^2}{(3-4\eta L)K}\\
    &\quad + \left(\frac{2(1+\delta)(2-\delta)}{(1-\sqrt{1-\delta})^2}+\frac{1+\delta}{\delta}\right)  \frac{1-\delta}{(1-\sqrt{1-\delta})^2} \frac{4(\eta L+1)\eta^2L^2}{3-4\eta L} \frac{M^2+\sigma^2}{(1-\mu)^2}\\
    &\underset{\mu=0}{=}\frac{4[F(\tilde{\textbf{x}}_0)-F(\tilde{\textbf{x}}^*)]}{\eta(3-4\eta L)T} + \frac{2\eta L\sigma^2}{(3-4\eta L)K}\\
    &\quad + \left(\frac{2(1+\delta)(2-\delta)}{(1-\sqrt{1-\delta})^2}+\frac{1+\delta}{\delta}\right)  \frac{1-\delta}{(1-\sqrt{1-\delta})^2} \frac{4(\eta L+1)\eta^2L^2}{3-4\eta L}(M^2+\sigma^2)\,.\\
\end{split}
\end{equation}
\end{proof}

\begin{atheorem}\label{appendix:saef-sgd convergence}
If Assumptions \ref{appendix:lipschitz gradient}, \ref{appendix:bounded variance}, \ref{appendix:bounded second moment} and \ref{appendix:compressor} exist, and the learning rate $0<\eta_t=\eta<\frac{3}{2L}$ for all $t=0,\cdots,T-1$, for SAEF-SGD we have
\begin{equation}
\begin{split}
    \min_{t=0,\cdots,T-1}\mathbb{E}\|\nabla F(\textbf{x}^{(k)}_t)\|^2 &\leq \frac{4[F(\textbf{x}_0)-F(\textbf{x}^*)]}{\eta(3-2\eta L)T}+\frac{4\eta L\sigma^2}{(3-2\eta L)K}\\
    &\quad +\left(2+\frac{8(\frac{2(1+\delta)(2-\delta)}{(1-\sqrt{1-\delta})^2}+\frac{1+\delta}{\delta})}{3-2\eta L}\right)\frac{1-\delta}{(1-\sqrt{1-\delta})^2}\eta^2L^2(M^2+\sigma^2)\,.\\
\end{split}
\end{equation}
\end{atheorem}
\begin{proof}
Following Theorem \ref{appendix:saef-sgd auxiliary convergence}, take expectation at iteration $t$ and we have
\begin{equation}
\begin{split}
    \mathbb{E}F(\tilde{\textbf{x}}_{t+1}) &= F(\tilde{\textbf{x}}_t) \underbrace{- \langle \nabla F(\tilde{\textbf{x}}_t),\frac{\eta}{K}\sum^{K}_{k=1}\nabla F(\textbf{x}^{(k)}_{t+\frac{1}{2}})\rangle}_\text{\textcircled{1}} + \underbrace{\frac{\eta^2L}{2K^2}\mathbb{E}\|\sum^{K}_{k=1}\nabla f(\textbf{x}^{(k)}_{t+\frac{1}{2}};\xi^{(k)}_t)\|^2}_\text{\textcircled{2}}\,.\\
\end{split}
\end{equation}
Firstly we consider \textcircled{1}.
\begin{equation}
\begin{split}
    &-\langle \nabla F(\tilde{\textbf{x}}_t),\frac{\eta}{K}\sum^{K}_{k=1}\nabla F(\textbf{x}^{(k)}_{t+\frac{1}{2}})\rangle\\
    &= -\frac{\eta}{K}\sum^{K}_{k=1}\langle\nabla F(\tilde{\textbf{x}}_t)-\nabla F(\textbf{x}^{(k)}_{t+\frac{1}{2}}),\nabla F(\textbf{x}^{(k)}_{t+\frac{1}{2}})\rangle-\frac{\eta}{K}\sum^{K}_{k=1}\|\nabla F(\textbf{x}^{(k)}_{t+\frac{1}{2}})\|^2\\
    &\leq \frac{\eta}{K}\sum^{K}_{k=1}(\frac{\rho}{2}\|\nabla F(\textbf{x}^{(k)}_{t+\frac{1}{2}})\|^2+\frac{1}{2\rho}\|\nabla F(\tilde{\textbf{x}}_t)-\nabla F(\textbf{x}^{(k)}_{t+\frac{1}{2}})\|^2) - \frac{\eta}{K}\sum^{K}_{k=1}\|\nabla F(\textbf{x}^{(k)}_{t+\frac{1}{2}})\|^2\\
    &= -\frac{\eta}{K}(1-\frac{\rho}{2})\sum^{K}_{k=1}\|\nabla F(\textbf{x}^{(k)}_{t+\frac{1}{2}})\|^2+\frac{\eta}{2\rho K}\sum^{K}_{k=1}\|\nabla F(\tilde{\textbf{x}}_t)-\nabla F(\textbf{x}^{(k)}_{t+\frac{1}{2}})\|^2\\
    &\leq -\frac{\eta}{K}(1-\frac{\rho}{2})\sum^{K}_{k=1}\|\nabla F(\textbf{x}^{(k)}_{t+\frac{1}{2}})\|^2 + \frac{\eta L^2}{2\rho K}\sum^{K}_{k=1}\|\textbf{e}_t+\frac{1}{K}\sum^{K}_{k=1}\textbf{e}^{(k)}_t-\textbf{e}^{(k)}_t\|^2\\
\end{split}
\end{equation}
Then we consider \textcircled{2}.
\begin{equation}
\begin{split}
    \frac{\eta^2L}{2K^2}\mathbb{E}\|\sum^{K}_{k=1}\nabla f(\textbf{x}^{(k)}_{t+\frac{1}{2}};\xi^{(k)}_t)\|^2 &\leq \frac{\eta^2L}{2K^2}\mathbb{E}\|\sum^{K}_{k=1}(\nabla f(\textbf{x}^{(k)}_{t+\frac{1}{2}};\xi^{(k)}_t)-\nabla F(\textbf{x}^{(k)}_{t+\frac{1}{2}}))\|^2 + \frac{\eta^2L}{2K^2}\|\sum^{K}_{k=1}\nabla F(\textbf{x}^{(k)}_{t+\frac{1}{2}})\|^2\\
    &\leq \frac{\eta^2L}{2K^2}\sum^{K}_{k=1}\mathbb{E}\|\nabla f(\textbf{x}^{(k)}_{t+\frac{1}{2}};\xi^{(k)}_t)-\nabla F(\textbf{x}^{(k)}_{t+\frac{1}{2}})\|^2 + \frac{\eta^2L}{2K}\sum^{K}_{k=1}\|\nabla F(\textbf{x}^{(k)}_{t+\frac{1}{2}})\|^2\\
    &\leq \frac{\eta^2L\sigma^2}{2K} + \frac{\eta^2L}{2K}\sum^{K}_{k=1}\|\nabla F(\textbf{x}^{(k)}_{t+\frac{1}{2}})\|^2\\
\end{split}
\end{equation}
Combine them and we have
\begin{equation}
    \mathbb{E}F(\tilde{\textbf{x}}_{t+1})-F(\tilde{\textbf{x}}_t) \leq -\frac{\eta}{K}(1-\frac{\rho}{2}-\frac{\eta L}{2})\sum^{K}_{k=1}\|\nabla F(\textbf{x}^{(k)}_{t+\frac{1}{2}})\|^2 + \frac{\eta L^2}{2\rho K}\sum^{K}_{k=1}\|\textbf{e}_t+\frac{1}{K}\sum^{K}_{k=1}\textbf{e}^{(k)}_t-\textbf{e}^{(k)}_t\|^2 + \frac{\eta^2L\sigma^2}{2K}\,.
\end{equation}
Sum from $t=0$ to $T-1$ and take the total expectation,
\begin{equation}
\begin{split}
    &F(\tilde{\textbf{x}}^*)-F(\tilde{\textbf{x}}_0) \leq \mathbb{E}(\sum^{T-1}_{t=0}(\mathbb{E}F(\tilde{\textbf{x}}_{t+1})-F(\tilde{\textbf{x}}_t)))\\
    &\leq -\eta(1-\frac{\rho}{2}-\frac{\eta L}{2})\frac{1}{K}\sum^{T-1}_{t=0}\sum^{K}_{k=1}\mathbb{E}\|\nabla F(\textbf{x}^{(k)}_{t+\frac{1}{2}})\|^2 + \frac{\eta L^2}{2\rho K}\sum^{T-1}_{t=0}\sum^{K}_{k=1}\mathbb{E}\|\textbf{e}_t+\frac{1}{K}\sum^{K}_{k=1}\textbf{e}^{(k)}_t-\textbf{e}^{(k)}_t\|^2 +\frac{\eta^2L\sigma^2T}{2K}\,.\\
\end{split}
\end{equation}
Note that $\tilde{\textbf{x}}^*=\textbf{x}^*$ and $\tilde{\textbf{x}}_0=\textbf{x}_0$. After rearranging,
\begin{equation}
\begin{split}
    \frac{1}{KT}\sum^{T-1}_{t=0}\sum^{K}_{k=1}\mathbb{E}\|\nabla F(\textbf{x}^{(k)}_{t+\frac{1}{2}})\|^2 &\leq \frac{F(\textbf{x}_0)-F(\textbf{x}^*)}{\eta(1-\frac{\rho}{2}-\frac{\eta L}{2})T} + \frac{L^2}{2\rho KT(1-\frac{\rho}{2}-\frac{\eta L}{2})}\sum^{T-1}_{t=0}\sum^{K}_{k=1}\mathbb{E}\|\textbf{e}_t+\frac{1}{K}\sum^{K}_{k=1}\textbf{e}^{(k)}_t-\textbf{e}^{(k)}_t\|^2\\
    &\quad + \frac{\eta L\sigma^2}{2(1-\frac{\rho}{2}-\frac{\eta L}{2})K}\,.\\
\end{split}
\end{equation}
Now we consider $\|\nabla F(\textbf{x}^{(k)}_t)\|$.
\begin{equation}
\begin{split}
    \|\nabla F(\textbf{x}^{(k)}_t)\|^2& = \|\nabla F(\textbf{x}^{(k)}_t)-\nabla F(\textbf{x}^{(k)}_{t+\frac{1}{2}})+\nabla F(\textbf{x}^{(k)}_{t+\frac{1}{2}})\|^2\\
    &\leq 2\|\nabla F(\textbf{x}^{(k)}_t)-\nabla F(\textbf{x}^{(k)}_{t+\frac{1}{2}})\|^2+2\|\nabla F(\textbf{x}^{(k)}_{t+\frac{1}{2}})\|^2\\
    &\leq 2L^2\|\textbf{e}^{(k)}_t\|^2 + 2\|\nabla F(\textbf{x}^{(k)}_{t+\frac{1}{2}})\|^2
\end{split}
\end{equation}
Thus,
\begin{equation}
\begin{split}
    \min_{t=0,\cdots,T-1}\mathbb{E}\|\nabla F(\textbf{x}^{(k)}_t)\|^2 &\leq\frac{1}{T}\sum^{T-1}_{t=0}\mathbb{E}\|\nabla F(\textbf{x}^{(k)}_t)\|^2 = \frac{1}{KT}\sum^{T-1}_{t=0}\sum^{K}_{k=1}\mathbb{E}\|\nabla F(\textbf{x}^{(k)}_t)\|^2\\
    &\leq \frac{2L^2}{KT}\sum^{T-1}_{t=0}\sum^{K}_{k=1}\mathbb{E}\|\textbf{e}^{(k)}_t\|^2 + \frac{2}{KT}\sum^{T-1}_{t=0}\sum^{K}_{k=1}\mathbb{E}\|\nabla F(\textbf{x}^{(k)}_{t+\frac{1}{2}})\|^2\\
    &\leq \frac{2[F(\textbf{x}_0)-F(\textbf{x}^*)]}{\eta(1-\frac{\rho}{2}-\frac{\eta L}{2})T} + \frac{\eta L\sigma^2}{(1-\frac{\rho}{2}-\frac{\eta L}{2})K} + \frac{2L^2}{KT}\sum^{T-1}_{t=0}\sum^{K}_{k=1}\mathbb{E}\|\textbf{e}^{(k)}_t\|^2\\
    &\quad +\frac{L^2}{\rho KT(1-\frac{\rho}{2}-\frac{\eta L}{2})}\sum^{T-1}_{t=0}\sum^{K}_{k=1}\mathbb{E}\|\textbf{e}_t+\frac{1}{K}\sum^{K}_{k=1}\textbf{e}^{(k)}_t-\textbf{e}^{(k)}_t\|^2\\
    &\leq \frac{2[F(\textbf{x}_0)-F(\textbf{x}^*)]}{\eta(1-\frac{\rho}{2}-\frac{\eta L}{2})T} + \frac{\eta L\sigma^2}{(1-\frac{\rho}{2}-\frac{\eta L}{2})K} + 2L^2\frac{1-\delta}{(1-\sqrt{1-\delta})^2}\frac{\eta^2(M^2+\sigma^2)}{(1-\mu)^2}\\
    &\quad+ \frac{L^2}{\rho(1-\frac{\rho}{2}-\frac{\eta L}{2})}(\frac{2(1+\delta)(2-\delta)}{(1-\sqrt{1-\delta})^2}+\frac{1+\delta}{\delta})\frac{1-\delta}{(1-\sqrt{1-\delta})^2}\frac{\eta^2(M^2+\sigma^2)}{(1-\mu)^2}\\
    &\leq \frac{2[F(\textbf{x}_0)-F(\textbf{x}^*)]}{\eta(1-\frac{\rho}{2}-\frac{\eta L}{2})T} + \frac{\eta L\sigma^2}{(1-\frac{\rho}{2}-\frac{\eta L}{2})K}\\
    &\quad + \left(2+\frac{\frac{2(1+\delta)(2-\delta)}{(1-\sqrt{1-\delta})^2}+\frac{1+\delta}{\delta}}{\rho(1-\frac{\rho}{2}-\frac{\eta L}{2})}\right)\frac{1-\delta}{(1-\sqrt{1-\delta})^2}\frac{\eta^2L^2(M^2+\sigma^2)}{(1-\mu)^2}\,.\\
\end{split}
\end{equation}
where the fourth inequality follows Lemma \ref{(ee)^2} and \ref{appendix:mismatch bound}. Let $\rho=\frac{1}{2}$ (or some other appropriate value) for ease of comparison with existing works (next corollary),
\begin{equation}
\begin{split}
    \min_{t=0,\cdots,T-1}\mathbb{E}\|\nabla F(\textbf{x}^{(k)}_t)\|^2 &\leq \frac{4[F(\textbf{x}_0)-F(\textbf{x}^*)]}{\eta(3-2\eta L)T}+\frac{4\eta L\sigma^2}{(3-2\eta L)K}\\
    &\quad +\left(2+\frac{8(\frac{2(1+\delta)(2-\delta)}{(1-\sqrt{1-\delta})^2}+\frac{1+\delta}{\delta})}{3-2\eta L}\right)\frac{1-\delta}{(1-\sqrt{1-\delta})^2}\frac{\eta^2L^2(M^2+\sigma^2)}{(1-\mu)^2}\\
    &\underset{\mu=0}{=}\frac{4[F(\textbf{x}_0)-F(\textbf{x}^*)]}{\eta(3-2\eta L)T}+\frac{4\eta L\sigma^2}{(3-2\eta L)K}\\
    &\quad +\left(2+\frac{8(\frac{2(1+\delta)(2-\delta)}{(1-\sqrt{1-\delta})^2}+\frac{1+\delta}{\delta})}{3-2\eta L}\right)\frac{1-\delta}{(1-\sqrt{1-\delta})^2}\eta^2L^2(M^2+\sigma^2)\,.\\
\end{split}
\end{equation}
\end{proof}

\begin{acorollary}
Under the same conditions of Theorem \ref{appendix:saef-sgd convergence}, the compression error term
\begin{equation}
    \left(2+\frac{8(\frac{2(1+\delta)(2-\delta)}{(1-\sqrt{1-\delta})^2}+\frac{1+\delta}{\delta})}{3-2\eta L}\right) \frac{1-\delta}{(1-\sqrt{1-\delta})^2}\eta^2L^2 (M^2+\sigma^2)
\end{equation}
in the upper bound of Theorem \ref{appendix:saef-sgd convergence} is much tighter than the corresponding EF-SGD compression error term
\begin{equation}
    \frac{32L^2(1-\delta)(M^2+\sigma^2)}{\delta^2}(1+\frac{16}{\delta^2})\frac{\eta^2}{3-2\eta L}\,.
\end{equation}
\end{acorollary}
\begin{proof}
We only need to show that
\begin{equation}
    h_2(\delta)=\frac{4}{\delta^2}(1+\frac{16}{\delta^2}) - [\frac{3}{4}+\frac{2(1+\delta)(2-\delta)}{(1-\sqrt{1-\delta})^2}+\frac{1+\delta}{\delta}] >0 \,,
\end{equation}
which is validated by the figure below that $h_2(\delta)\gg 0$, especially when $\delta\rightarrow 0$.
\begin{figure}[htb!]
    \centering
    \includegraphics[width=.3\textwidth]{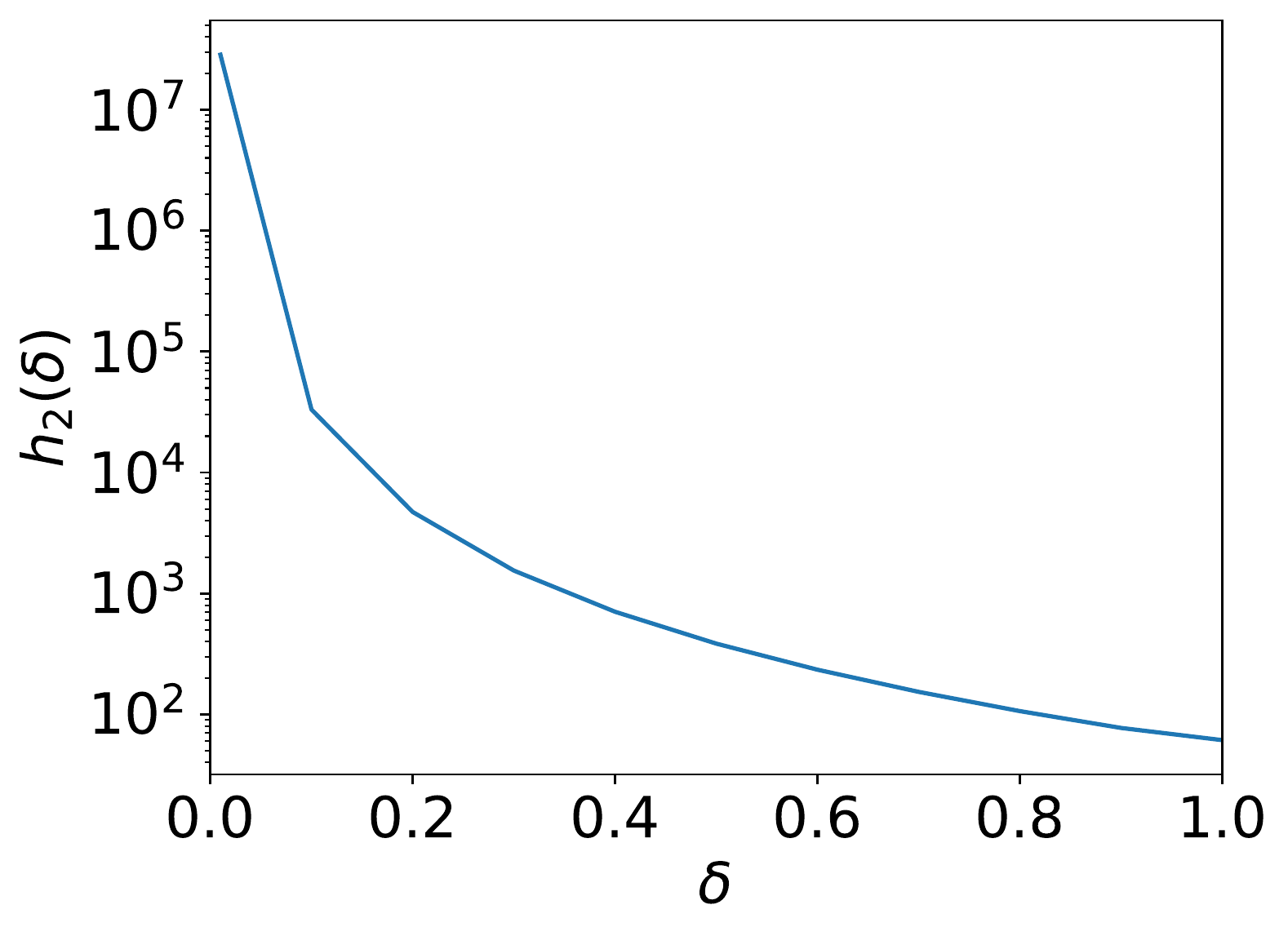}
    \caption{$h_2(\delta).$}
    \label{fig:h2}
\end{figure}
\end{proof}

\begin{acorollary}\label{appendix:saef-sgd convergence rate}
Under the same conditions of Theorem \ref{appendix:saef-sgd convergence}, let the learning rate $\eta<\frac{c\sqrt{K}}{\sqrt{T}}$, where $c>0$ is some constant. Then the convergence rate of $\textbf{x}^{(k)}_t$ in SAEF-SGD satisfies
\begin{equation}
    \min_{t=0,\cdots,T-1}\mathbb{E}\|\nabla F(\textbf{x}^{(k)}_t)\|^2 =\mathcal{O}(\frac{1}{\sqrt{KT}})\,.
\end{equation}
\end{acorollary}
\begin{proof}
For $T\geq c^2L^2K$,
\begin{equation}
    3-2\eta L\geq1\,.
\end{equation}
Thus,
\begin{equation}
\begin{split}
    \min_{t=0,\cdots,T-1}\mathbb{E}\|F(\textbf{x}^{(k)}_t)\|^2 &\leq \frac{4[F(\textbf{x}_0)-F(\textbf{x}^*)]}{\eta T}+\frac{\eta^2L\sigma^2}{K}\\
    &\quad +\left(1+8(\frac{2(1+\delta)(2-\delta)}{(1-\sqrt{1-\delta})^2}+\frac{1+\delta}{\delta})\right)\frac{1-\delta}{(1-\sqrt{1-\delta})^2}\eta^2L^2(M^2+\sigma^2)\,.\\
    &=\mathcal{O}(\frac{1}{\sqrt{KT}})+\mathcal{O}(\frac{K}{T})\underset{T\geq K^3}{=}\mathcal{O}(\frac{1}{\sqrt{KT}})\,.\\
\end{split}
\end{equation}
\end{proof}

\section{Convergence of SAEF-SGD with Momentum}
Suppose the learning rate $\eta_t=\eta$ for $t=0,\cdots,T-1$. We begin with the definition of two virtual variables $\textbf{z}_t$ and $\textbf{p}_t$, where
\begin{equation}
\textbf{p}_t=\begin{cases}
\frac{\mu}{1-\mu}(\tilde{\textbf{x}}_t-\tilde{\textbf{x}}_{t-1}),\quad t\geq 1\\
\textbf{0},\quad t=0\\
\end{cases}
\end{equation}
and
\begin{equation}
    \textbf{z}_t=\tilde{\textbf{x}}_t+\textbf{p}_t\,.
\end{equation}
The update of $\textbf{z}_t$ satisfies
\begin{equation}
\begin{split}
    \textbf{z}_{t+1}-\textbf{z}_t &= (\tilde{\textbf{x}}_{t+1}-\tilde{\textbf{x}}_t) + \frac{\mu}{1-\mu}(\tilde{\textbf{x}}_{t+1}-\tilde{\textbf{x}}_t)-\frac{\mu}{1-\mu}(\tilde{\textbf{x}}_t-\tilde{\textbf{x}}_{t-1})\\
    &=-\frac{\eta}{K}\sum^{K}_{k=1}\textbf{m}^{(k)}_{t+1}-\frac{\mu}{1-\mu}\frac{\eta}{K}\sum^{K}_{k=1}\textbf{m}^{(k)}_{t+1}+\frac{\mu}{1-\mu}\frac{\eta}{K}
    \sum^{K}_{k=1}\textbf{m}^{(k)}_t\\
    &=-\frac{\eta}{(1-\mu)K}\sum^{K}_{k=1}(\textbf{m}^{(k)}_{t+1}-\mu \textbf{m}^{(k)}_t)\\
    &=-\frac{\eta}{(1-\mu)K}\sum^{K}_{k=1}\nabla f(\textbf{x}^{(k)}_{t+\frac{1}{2}};\xi^{(k)}_t)\,.\\
\end{split}
\end{equation}

\begin{atheorem}\label{appendix:saef-sgdm convergence}
If Assumption \ref{appendix:lipschitz gradient}, \ref{appendix:bounded variance}, \ref{appendix:bounded second moment} and \ref{appendix:compressor} exist, and the learning rate $0<\eta_t=\eta$ satisfies $\alpha\coloneqq 1-\frac{\eta L}{1-\mu}-\frac{2\mu^2\eta^2L^2}{(1-\mu)^4}>0$ for all $t=0,\cdots,T-1$, for SAEF-SGD with momentum we have
\begin{equation}
\begin{split}
    \min_{t=0,\cdots,T-1}\mathbb{E}\|\nabla F(\textbf{x}^{(k)}_t)\|^2 &\leq \frac{4(1-\mu)[F(\textbf{x}_0)-F(\textbf{x}^*)]}{\alpha\eta T} + \frac{2\left(1+\frac{2\mu^2\eta L}{(1-\mu)^3}\right)\eta L\sigma^2}{\alpha(1-\mu)K}\\
    &\quad + \left(\frac{8(1+\delta)(2-\delta)}{(1-\sqrt{1-\delta})^2}+\frac{4(1+\delta)}{\delta}+2\alpha\right)\frac{1-\delta}{(1-\sqrt{1-\delta})^2}\frac{\eta^2L^2(M^2+\sigma^2)}{\alpha(1-\mu)^2}\,,\\
\end{split}
\end{equation}
\begin{equation}
\begin{split}
    \min_{t=0,\cdots,T-1}\mathbb{E}\|\nabla F(\tilde{\textbf{x}}_t)\|^2 &\leq \frac{4(1-\mu)[F(\textbf{x}_0)-F(\textbf{x}^*)]}{\alpha\eta T} + \frac{2\left(1+\frac{2\mu^2\eta L}{(1-\mu)^3}\right)\eta L\sigma^2}{\alpha(1-\mu)K}\\
    &\quad +(\frac{4}{\alpha}+2)\left(\frac{2(1+\delta)(2-\delta)}{(1-\sqrt{1-\delta})^2}+\frac{1+\delta}{\delta}\right)\frac{1-\delta}{(1-\sqrt{1-\delta})^2}\frac{\eta^2L^2(M^2+\sigma^2)}{(1-\mu)^2}\,.\\
\end{split}
\end{equation}
\end{atheorem}
\begin{proof}
\begin{equation}
\begin{split}
    F(\textbf{z}_{t+1}) &\leq F(\textbf{z}_t)+\langle\nabla F(\textbf{z}_t), \textbf{z}_{t+1}-\textbf{z}_t\rangle+\frac{L}{2}\|\textbf{z}_{t+1}-\textbf{z}_t\|^2\\
    &=F(\textbf{z}_t)-\frac{\eta}{1-\mu}\langle \nabla F(\textbf{z}_t),\frac{1}{K}\sum^{K}_{k=1}\nabla f(\textbf{x}^{(k)}_{t+\frac{1}{2}};\xi^{(k)}_t)\rangle + \frac{\eta^2L}{2(1-\mu)^2K^2}\|\sum^{K}_{k=1}\nabla f(\textbf{x}^{(k)}_{t+\frac{1}{2}};\xi^{(k)}_t)\|^2\\
\end{split}
\end{equation}
Take expectation at iteration $t$,
\begin{equation}
\begin{split}
    \mathbb{E}F(\textbf{z}_{t+1}) \leq F(\textbf{z}_t)\underbrace{-\frac{\eta}{1-\mu}\langle\nabla F(\textbf{z}_t),\frac{1}{K}\sum^{K}_{k=1}\nabla F(\textbf{x}^{(k)}_{t+\frac{1}{2}})\rangle}_\text{\textcircled{1}}+\underbrace{\frac{\eta^2L}{2(1-\mu)^2K^2}\mathbb{E}\|\sum^{K}_{k=1}\nabla f(\textbf{x}^{(k)}_{t+\frac{1}{2}};\xi^{(k)}_t)\|^2}_\text{\textcircled{2}}\,.
\end{split}
\end{equation}
Firstly we consider \textcircled{1}.
\begin{equation}
\begin{split}
    &-\frac{\eta}{1-\mu}\langle\nabla F(\textbf{z}_t),\frac{1}{K}\sum^{K}_{k=1}\nabla F(\textbf{x}^{(k)}_{t+\frac{1}{2}})\rangle\\
    &= -\frac{\eta}{(1-\mu)K}\sum^{K}_{k=1}\langle\nabla F(\textbf{z}_t)-\nabla F(\textbf{x}^{(k)}_{t+\frac{1}{2}})+\nabla F(\textbf{x}^{(k)}_{t+\frac{1}{2}}),\nabla F(\textbf{x}^{(k)}_{t+\frac{1}{2}})\rangle\\
    &=-\frac{\eta}{(1-\mu)K}\sum^{K}_{k=1}\langle\nabla F(\textbf{z}_t)-\nabla F(\textbf{x}^{(k)}_{t+\frac{1}{2}}),\nabla F(\textbf{x}^{(k)}_{t+\frac{1}{2}})\rangle - \frac{\eta}{(1-\mu)K}\sum^{K}_{k=1}\|\nabla F(\textbf{x}^{(k)}_{t+\frac{1}{2}})\|^2\\
    &\leq -\frac{\eta}{(1-\mu)K}(1-\frac{\rho}{2})\sum^{K}_{k=1}\|\nabla F(\textbf{x}^{(k)}_{t+\frac{1}{2}})\|^2 + \frac{\eta}{(1-\mu)K}\frac{1}{2\rho}\sum^{K}_{k=1}\|\nabla F(\textbf{z}_t)-\nabla F(\textbf{x}^{(k)}_{t+\frac{1}{2}})\|^2\\
    &\leq -\frac{\eta}{(1-\mu)K}(1-\frac{\rho}{2})\sum^{K}_{k=1}\|\nabla F(\textbf{x}^{(k)}_{t+\frac{1}{2}})\|^2 + \frac{\eta L^2}{(1-\mu)K}\frac{1}{2\rho}\sum^{K}_{k=1}\|\tilde{\textbf{x}}_t+\textbf{p}_t-\textbf{x}^{(k)}_{t+\frac{1}{2}}\|^2\\
    &\leq -\frac{\eta}{(1-\mu)K}(1-\frac{\rho}{2})\sum^{K}_{k=1}\|\nabla F(\textbf{x}^{(k)}_{t+\frac{1}{2}})\|^2 + \frac{\eta L^2}{(1-\mu)K}\frac{1}{\rho}\sum^{K}_{k=1}(\|\textbf{e}_t+\sum^{K}_{k=1}\textbf{e}^{(k)}_t-\textbf{e}^{(k)}_t\|^2+\|\textbf{p}_t\|^2)\\
\end{split}
\end{equation}
Then we consider \textcircled{2} following Theorem \ref{appendix:saef-sgd convergence}.
\begin{equation}
\begin{split}
    \frac{\eta^2L}{2(1-\mu)^2K^2}\mathbb{E}\|\sum^{K}_{k=1}\nabla f(\textbf{x}^{(k)}_{t+\frac{1}{2}};\xi^{(k)}_t)\|^2 &\leq \frac{\eta^2L\sigma^2}{2(1-\mu)^2K} + \frac{\eta^2L}{2(1-\mu)^2K}\sum^{K}_{k=1}\|\nabla F(\textbf{x}^{(k)}_{t+\frac{1}{2}})\|^2\\    
\end{split}
\end{equation}
Combine them and we have
\begin{equation}
\begin{split}
    \mathbb{E}F(\textbf{z}_{t+1})-F(\textbf{z}_t) &\leq -\frac{\eta}{(1-\mu)K}(1-\frac{\rho}{2}-\frac{\eta L}{2(1-\mu)})\sum^{K}_{k=1}\|\nabla F(\textbf{x}^{(k)}_{t+\frac{1}{2}})\|^2 + \frac{\eta^2L\sigma^2}{2(1-\mu)^2K}\\
    &\quad+ \frac{\eta L^2}{(1-\mu)K}\frac{1}{\rho}\sum^{K}_{k=1}(\|\textbf{e}_t+\sum^{K}_{k=1}\textbf{e}^{(k)}_t-\textbf{e}^{(k)}_t\|^2+\|\textbf{p}_t\|^2)\,.\\
\end{split}
\end{equation}
Sum from $t=0$ to $T-1$ and take the total expectation,
\begin{equation}
\begin{split}
    F(\textbf{z}^*)-F(\textbf{z}_0) &\leq \mathbb{E}(\sum^{T-1}_{t=0}(\mathbb{E}F(\textbf{z}_{t+1})-F(\textbf{z}_t)))\\
    &\leq -\frac{\eta}{(1-\mu)K}(1-\frac{\rho}{2}-\frac{\eta L}{2(1-\mu)})\sum^{T-1}_{t=0}\sum^{K}_{k=1}\mathbb{E}\|\nabla F(\textbf{x}^{(k)}_{t+\frac{1}{2}})\|^2 + \frac{\eta^2L\sigma^2T}{2(1-\mu)^2K}\\
    &\quad+ \frac{\eta L^2}{(1-\mu)K}\frac{1}{\rho}\sum^{T-1}_{t=0}\sum^{K}_{k=1}(\mathbb{E}\|\textbf{e}_t+\sum^{K}_{k=1}\textbf{e}^{(k)}_t-\textbf{e}^{(k)}_t\|^2+\mathbb{E}\|\textbf{p}_t\|^2)\,.\\
\end{split}
\end{equation}
Now we consider $\sum^{T-1}_{t=0}\mathbb{E}\|\textbf{p}_t\|^2$.
\begin{equation}
\begin{split}
    \|\textbf{p}_t\|^2 &=\frac{\mu^2}{(1-\mu)^2}\|\tilde{\textbf{x}}_t-\tilde{\textbf{x}}_{t-1}\|^2=\frac{\mu^2\eta^2}{(1-\mu)^2K^2}\|\sum^{K}_{k=1}\textbf{m}^{(k)}_t\|^2\\
    &=\frac{\mu^2\eta^2}{(1-\mu)^2K^2}\|\sum^{K}_{k=1}\sum^{t-1}_{i=0}\mu^{t-1-i}\nabla f(\textbf{x}^{(k)}_{i+\frac{1}{2}};\xi^{(k)}_i)\|^2\\
    &\leq \frac{\mu^2\eta^2}{(1-\mu)^2K^2}(\sum^{t-1}_{i=0}\mu^{t-1-i})\sum^{t-1}_{i=0}\mu^{t-1-i}\|\sum^{K}_{k=1}\nabla f(\textbf{x}^{(k)}_{i+\frac{1}{2}};\xi^{(k)}_i)\|^2\\
    &\leq \frac{\mu^2\eta^2}{(1-\mu)^3K^2}\sum^{t-1}_{i=0}\mu^{t-1-i}\|\sum^{K}_{k=1}\nabla f(\textbf{x}^{(k)}_{i+\frac{1}{2}};\xi^{(k)}_i)\|^2\\
\end{split}
\end{equation}
\begin{equation}
\begin{split}
    \sum^{T-1}_{t=0}\mathbb{E}\|\textbf{p}_t\|^2 &\leq \frac{\mu^2\eta^2}{(1-\mu)^3K^2}\sum^{T-1}_{t=0}\sum^{t-1}_{i=0}\mu^{t-1-i}\mathbb{E}\|\sum^{K}_{k=1}\nabla f(\textbf{x}^{(k)}_{i+\frac{1}{2}};\xi^{(k)}_i)\|^2\\
    &= \frac{\mu^2\eta^2}{(1-\mu)^3K^2}\sum^{T-2}_{i=0}\sum^{T-1}_{t=i+1}\mu^{t-1-i}\mathbb{E}\|\sum^{K}_{k=1}\nabla f(\textbf{x}^{(k)}_{i+\frac{1}{2}};\xi^{(k)}_i)\|^2\\
    &\leq \frac{\mu^2\eta^2}{(1-\mu)^4K^2}\sum^{T-2}_{i=0}\mathbb{E}\|\sum^{K}_{k=1}\nabla f(\textbf{x}^{(k)}_{i+\frac{1}{2}};\xi^{(k)}_i)\|^2\\
    &\leq \frac{\mu^2\eta^2}{(1-\mu)^4K^2} \sum^{T-1}_{t=0}[\mathbb{E}\|\sum^{K}_{k=1}(\nabla f(\textbf{x}^{(k)}_{t+\frac{1}{2}};\xi^{(k)}_t)-\nabla F(\textbf{x}^{(k)}_{t+\frac{1}{2}}))\|^2+\mathbb{E}\|\sum^{K}_{k=1}\nabla F(\textbf{x}^{(k)}_{t+\frac{1}{2}})\|^2]\\
    &\leq \frac{\mu^2\eta^2\sigma^2T}{(1-\mu)^4K} + \frac{\mu^2\eta^2}{(1-\mu)^4K}\sum^{T-1}_{t=0}\sum^{K}_{k=1}\mathbb{E}\|\nabla F(\textbf{x}^{(k)}_{t+\frac{1}{2}})\|^2\\
\end{split}
\end{equation}
Thus,
\begin{equation}
\begin{split}
    F(\textbf{z}^*)-F(\textbf{z}_0) &\leq -\frac{\eta}{(1-\mu)K}\left(1-\frac{\rho}{2}-\frac{\eta L}{2(1-\mu)}-\frac{\mu^2\eta^2L^2}{\rho(1-\mu)^4}\right)\sum^{T-1}_{t=0}\sum^{K}_{k=1}\mathbb{E}\|\nabla F(\textbf{x}^{(k)}_{t+\frac{1}{2}})\|^2\\
    &\quad + \left(\frac{1}{2}+\frac{\mu^2\eta L}{\rho(1-\mu)^3}\right)\frac{\eta^2L\sigma^2T}{(1-\mu)^2K} + \frac{\eta L^2}{\rho(1-\mu)K}\sum^{T-1}_{t=0}\sum^{K}_{k=1}\mathbb{E}\|\textbf{e}_t+\frac{1}{K}\sum^{K}_{k=1}\textbf{e}^{(k)}_t-\textbf{e}^{(k)}_t\|^2\,.\\
\end{split}
\end{equation}
Note that $\textbf{z}^*=\textbf{x}^*$ and $\textbf{z}_0=\textbf{x}_0$. Let $\rho=1$ (or some other appropriate value). After rearranging,
\begin{equation}
\begin{split}
    \frac{1}{KT}\sum^{T-1}_{t=0}\sum^{K}_{k=1}\mathbb{E}\|\nabla F(\textbf{x}^{(k)}_{t+\frac{1}{2}})\|^2 &\leq \frac{2(1-\mu)[F(\textbf{x}_0)-F(\textbf{x}^*)]}{\left(1-\frac{\eta L}{1-\mu}-\frac{2\mu^2\eta^2L^2}{(1-\mu)^4}\right)\eta T} + \frac{\left(1+\frac{2\mu^2\eta L}{(1-\mu)^3}\right)\eta L\sigma^2}{\left(1-\frac{\eta L}{1-\mu}-\frac{2\mu^2\eta^2L^2}{(1-\mu)^4}\right)(1-\mu)K}\\
    &\quad +\frac{2L^2}{\left(1-\frac{\eta L}{1-\mu}-\frac{2\mu^2\eta^2L^2}{(1-\mu)^4}\right)KT}\sum^{T-1}_{t=0}\sum^{K}_{k=1}\mathbb{E}\|\textbf{e}_t+\frac{1}{K}\sum^{K}_{k=1}\textbf{e}^{(k)}_t-\textbf{e}^{(k)}_t\|^2\,.\\
\end{split}
\end{equation}
Because $\|\nabla F(\textbf{x}^{(k)}_t)\|^2 \leq 2L^2\|\textbf{e}^{(k)}_t\|^2 + 2\|\nabla F(\textbf{x}^{(k)}_{t+\frac{1}{2}})\|^2$ and $\|\nabla F(\tilde{\textbf{x}}_t)\|^2 \leq 2L^2\|\textbf{e}_t+\frac{1}{K}\sum^{K}_{k=1}\textbf{e}^{(k)}_t-\textbf{e}^{(k)}_t\|^2 + 2\|\nabla F(\textbf{x}^{(k)}_{t+\frac{1}{2}})\|^2$, we have
\begin{equation}
\begin{split}
    \min_{t=0,\cdots,T-1}\mathbb{E}\|\nabla F(\textbf{x}^{(k)}_t)\|^2 &\leq\frac{1}{T}\sum^{T-1}_{t=0}\mathbb{E}\|\nabla F(\textbf{x}^{(k)}_t)\|^2 = \frac{1}{KT}\sum^{T-1}_{t=0}\sum^{K}_{k=1}\mathbb{E}\|\nabla F(\textbf{x}^{(k)}_t)\|^2\\
    &\leq \frac{2L^2}{KT}\sum^{T-1}_{t=0}\sum^{K}_{k=1}\mathbb{E}\|\textbf{e}^{(k)}_t\|^2 + \frac{2}{KT}\sum^{T-1}_{t=0}\sum^{K}_{k=1}\mathbb{E}\|\nabla F(\textbf{x}^{(k)}_{t+\frac{1}{2}})\|^2\\
    &\leq \frac{4(1-\mu)[F(\textbf{x}_0)-F(\textbf{x}^*)]}{\alpha\eta T} + \frac{2\left(1+\frac{2\mu^2\eta L}{(1-\mu)^3}\right)\eta L\sigma^2}{\alpha(1-\mu)K}\\
    &\quad + \frac{4L^2}{\alpha KT}\sum^{T-1}_{t=0}\sum^{K}_{k=1}\mathbb{E}\|\textbf{e}_t+\frac{1}{K}\sum^{K}_{k=1}\textbf{e}^{(k)}_t-\textbf{e}^{(k)}_t\|^2 +  \frac{2L^2}{KT}\sum^{T-1}_{t=0}\sum^{K}_{k=1}\mathbb{E}\|\textbf{e}^{(k)}_t\|^2\\
    &\leq \frac{4(1-\mu)[F(\textbf{x}_0)-F(\textbf{x}^*)]}{\alpha\eta T} + \frac{2\left(1+\frac{2\mu^2\eta L}{(1-\mu)^3}\right)\eta L\sigma^2}{\alpha(1-\mu)K}\\
    &\quad + \left(\frac{8(1+\delta)(2-\delta)}{(1-\sqrt{1-\delta})^2}+\frac{4(1+\delta)}{\delta}+2\alpha\right)\frac{1-\delta}{(1-\sqrt{1-\delta})^2}\frac{\eta^2L^2(M^2+\sigma^2)}{\alpha(1-\mu)^2}\,,\\
\end{split}
\end{equation}
\begin{equation}
\begin{split}
    \min_{t=0,\cdots,T-1}\mathbb{E}\|\nabla F(\tilde{\textbf{x}}_t)\|^2 &\leq\frac{1}{T}\sum^{T-1}_{t=0}\mathbb{E}\|\nabla F(\tilde{\textbf{x}}_t)\|^2 = \frac{1}{KT}\sum^{T-1}_{t=0}\sum^{K}_{k=1}\mathbb{E}\|\nabla F(\tilde{\textbf{x}}_t)\|^2\\
    &\leq \frac{2L^2}{KT}\sum^{T-1}_{t=0}\sum^{K}_{k=1}\mathbb{E}\|\textbf{e}_t+\frac{1}{K}\sum^{K}_{k=1}\textbf{e}^{(k)}_t-\textbf{e}^{(k)}\|^2 + \frac{2}{KT}\sum^{T-1}_{t=0}\sum^{K}_{k=1}\mathbb{E}\|\nabla F(\textbf{x}^{(k)}_{t+\frac{1}{2}})\|^2\\
    &\leq \frac{4(1-\mu)[F(\textbf{x}_0)-F(\textbf{x}^*)]}{\alpha\eta T} + \frac{2\left(1+\frac{2\mu^2\eta L}{(1-\mu)^3}\right)\eta L\sigma^2}{\alpha(1-\mu)K}\\
    &\quad + \frac{2(2+\alpha)L^2}{\alpha KT}\sum^{T-1}_{t=0}\sum^{K}_{k=1}\mathbb{E}\|\textbf{e}_t+\frac{1}{K}\sum^{K}_{k=1}\textbf{e}^{(k)}_t-\textbf{e}^{(k)}_t\|^2\\
    &\leq \frac{4(1-\mu)[F(\textbf{x}_0)-F(\textbf{x}^*)]}{\alpha\eta T} + \frac{2\left(1+\frac{2\mu^2\eta L}{(1-\mu)^3}\right)\eta L\sigma^2}{\alpha(1-\mu)K}\\
    &\quad +(\frac{4}{\alpha}+2)\left(\frac{2(1+\delta)(2-\delta)}{(1-\sqrt{1-\delta})^2}+\frac{1+\delta}{\delta}\right)\frac{1-\delta}{(1-\sqrt{1-\delta})^2}\frac{\eta^2L^2(M^2+\sigma^2)}{(1-\mu)^2}\,.\\
\end{split}
\end{equation}
\end{proof}

\begin{acorollary}\label{appendix:saef-sgdm convergence rate}
Under the same conditions of Theorem \ref{appendix:saef-sgdm convergence}, let the learning rate $\eta<\frac{c\sqrt{K}}{\sqrt{T}}$, where $c>0$ is some constant. Then the convergence rate of $\textbf{x}^{(k)}_t$ in SAEF-SGD with momentum satisfies
\begin{equation}
    \min_{t=0,\cdots,T-1}\mathbb{E}\|\nabla F(\textbf{x}^{(k)}_t)\|^2 =\mathcal{O}(\frac{1}{\sqrt{KT}})\,.
\end{equation}
\end{acorollary}
\begin{proof}
Similar as the proof of Corollary \ref{appendix:saef-sgd convergence rate}.
\end{proof}

\end{document}